\documentclass{article}

 \usepackage[main, final]{neurips_2026}

\usepackage[utf8]{inputenc} 
\usepackage[T1]{fontenc}    
\usepackage{hyperref}       
\usepackage{url}            
\usepackage{booktabs}       
\usepackage{amsfonts}       
\usepackage{nicefrac}       
\usepackage{microtype}      
\usepackage{xcolor}         
\usepackage{amsmath}
\usepackage{amssymb}
\usepackage{amsthm}
\usepackage{algorithm}
\usepackage{algorithmic}
\usepackage{xcolor}
\usepackage{graphicx}
\usepackage{xspace}
\usepackage{subcaption}
\usepackage{multirow}
\usepackage{xcolor}
\usepackage{colortbl}
\usepackage{booktabs}
\usepackage{wrapfig}
\usepackage{xcolor}
\usepackage{colortbl}
\usepackage[table]{xcolor}

\title{\raisebox{-0.25\height}{\includegraphics[height=1.4em]{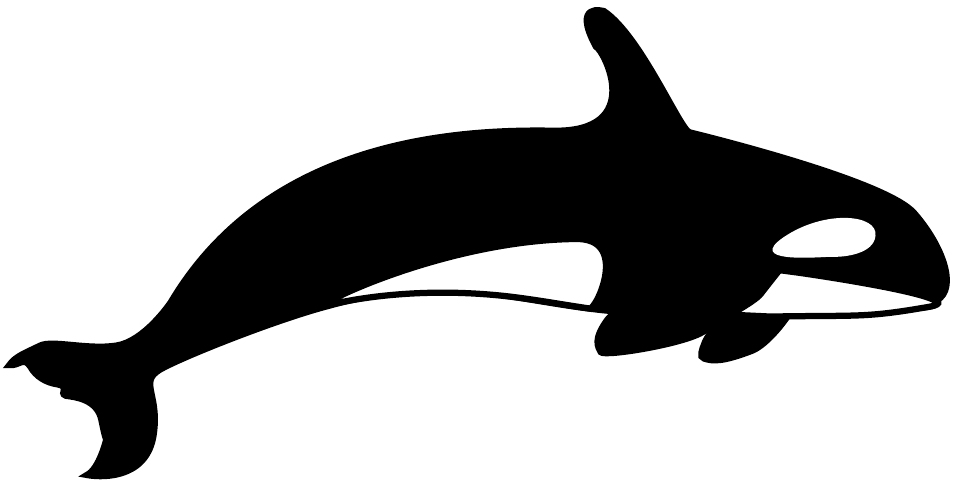}}\,%
       \textsc{Orca}: Hunting Compositional Failures in
       Text-to-Image Diffusion}

\theoremstyle{plain}
\newtheorem{theorem}{Theorem}
\newtheorem{proposition}[theorem]{Proposition}

\theoremstyle{definition}

\definecolor{rowVanilla}{HTML}{EEEEEE}   
\definecolor{rowREPA}   {HTML}{FBE3E3}   
\definecolor{rowREG}    {HTML}{FBE8D2}   
\definecolor{rowORCA}   {HTML}{E5F4E0}   
\definecolor{rowREPA}{HTML}{F7DADA}   
\definecolor{rowREG}{HTML}{FBE8D2}    
\definecolor{rowORCA}{HTML}{E5F4E0}   
\newcommand{\R}{\mathbb{R}}
\newcommand{\E}{\mathbb{E}}

\author{%
  \textbf{Arshia Hemmat}$^{1,2,3,*}$ \quad
  \textbf{Amirhossein Vahidi}$^{1,3,*}$ \quad
  \textbf{Amitis Shidani}$^{5,6}$ \\
  \textbf{Mohammad Vali Sanian}$^{1,7,8}$ \quad
  \textbf{Hesam Asadollahzadeh}$^{1,9}$ \quad
  \textbf{Aryan Yazdan Parast}$^{9}$ \\
  \textbf{Mo Lotfollahi}$^{1,2,3,4}$ \\[0.6em]
  $^1$Wellcome Sanger Institute \quad
  $^2$Cambridge Stem Cell Institute, University of Cambridge \\
  $^3$Cambridge Centre for AI in Medicine, University of Cambridge \\
  $^4$Department of Medicine, University of Cambridge \quad
  $^5$University of Oxford \quad
  $^6$Apple \\
  $^7$Department of Computer Science, University of Helsinki \\
  $^8$Institute for Molecular Medicine Finland (FIMM), University of Helsinki \\
  $^9$School of Computing and Information Systems, University of Melbourne \\[0.4em]
    {\small $^*$Equal contribution \quad Correspondence: \texttt{ml19@sanger.ac.uk} \quad Project: \href{https://arshiahemmat.com/orca}{\texttt{arshiahemmat.com/orca}}}
}

\begin{document}

\maketitle

\begin{abstract}
Text-to-image diffusion models fail predictably on compositional
prompts: attributes bind to the wrong objects, spatial relations
invert, and multi-object scenes lose count. Recent architectures
already augment CLIP with a T5 encoder precisely because CLIP's
contrastive embedding loses compositional structure, yet these
failures persist. We argue the binding problem is therefore not one
of missing information but of {misaligned} information: a text encoder
preserves compositional structure, but in a representation space
shaped by language modelling rather than vision, and the denoising
objective does not directly reward aligning the two. We show this correspondence can be supplied as an explicit training signal, that the relevant cross-modal information is concentrated in a low-rank
subspace of self-supervised visual features, and that supplying it
can be folded into diffusion training as a single auxiliary loss.
Our method, \textsc{Orca} (\textbf{O}rthogonal \textbf{R}esidual
\textbf{C}ompositional \textbf{A}lignment), aligns the latent of
a diffusion transformer with a low-rank target derived from a
frozen visual encoder, through a predictor whose orthogonal basis is parameterised by a learned residual between T5 and CLIP embeddings, which provides a prompt-dependent
signal for selecting the visual readout subspace. We prove that
the cross-modal information recoverable at a given rank is bounded
by the spectral mass of the visual encoder's covariance in the top
components. Across three diffusion-transformer backbones (DiT-B/2,
DiT-L/2, U-ViT-L), \textsc{Orca} improves FID and GenEval over
both vanilla and REPA baselines at zero inference-time cost; on
DiT-L/2 it reaches FID $\mathbf{16.65}$ and GenEval $\mathbf{0.291}$
at $200\text{K}$ steps, exceeding the strongest $400\text{K}$
baseline at half the training cost, with the largest gains
concentrated on attribute binding, spatial relations, and
multi-object prompts.
\end{abstract}
\begin{figure}
    \centering
    \includegraphics[width=\linewidth]{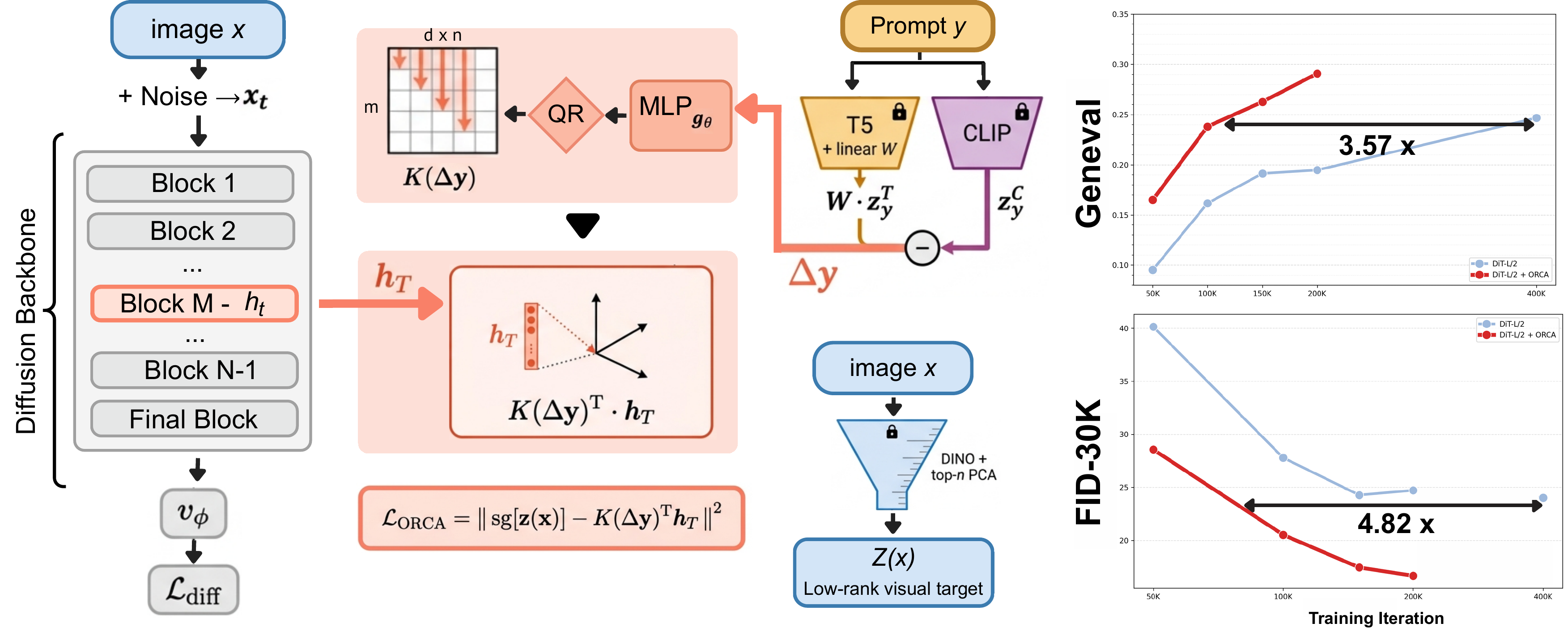}
\caption{{ORCA aligns the diffusion latent with a low-rank visual
target via a text-conditioned projector $K(\Delta y)$ built from the
T5$-$CLIP residual.} On DiT-L/2, ORCA reaches the vanilla $400\text{K}$
FID and GenEval in $\mathbf{4.8}{\times}$ and $\mathbf{3.6}{\times}$ fewer
steps, with zero inference overhead.}
\label{fig:demo_fig}
\end{figure}

\section{Introduction}
\label{sec:introduction}

Text-to-image diffusion models generate photorealistic images at scale, yet they fail predictably on prompts that require compositional reasoning. The fox is green when it should be red; the cat is on the mat instead of beside it; one of three apples is missing. These are not rendering failures, since the individual objects are crisp and well-formed, but conditioning failures, and they persist across model generations, architectures, and training scales~\citep{ ghosh2023geneval, compalign2025, infinitybeyond2025}.

A growing body of work has localised the source in the text encoder. CLIP's text encoder~\citep{radford2021learning}, the dominant conditioning signal in modern text-to-image models, behaves cross-modally as a bag-of-words: it captures which concepts appear in a caption but loses how those concepts are bound to one another~\citep{lewis2024doesclip}. Recent analyses trace this not to CLIP's training data but to the contrastive objective itself, which rewards alignment of shared content between modalities and provides no pressure to encode the syntactic structure that compositional generation requires~\citep{zarei2024improving,magnet2024}. State-of-the-art models including Stable Diffusion 3~\citep{esser2024scaling} and FLUX~\citep{{flux2024}} respond by augmenting CLIP with a T5 encoder~\citep{raffel2020exploring}, precisely because T5's language-modeling objective preserves the compositional structure CLIP discards.\looseness=-1

A separate line of work suggests where such grounding could come from. Diffusion transformers improve substantially when their internal representations are aligned with self-supervised visual encoders. REPA~\citep{yu2024representation} regularises diffusion hidden states toward DINOv2~\citep{oquab2024dinov2} features and reports more than $17\times$ faster convergence on class-conditional ImageNet; REG~\citep{reg2025} extends this to high-level class tokens; subsequent analysis identifies the spatial structure of self-supervised features as the source of the gain~\citep{wang2025whatmatters}. These methods supply the visual grounding signal the diffusion objective lacks, but they have been developed for class-conditional generation, where the conditioning is a single label and the binding problem does not arise. Whether the same mechanism addresses \emph{text-image} binding, where the conditioning is a structured sequence whose compositional content must be respected, remains open.


We propose ORCA, which closes this gap through three coupled observations. (1) The textual signal that should drive the alignment is not T5 itself but the residual between T5 and CLIP, the component of T5 not linearly expressible in CLIP's embedding space, which concentrates the compositional structure the contrastive encoder discards. (2) The visual signal the diffusion model should ground against is concentrated in a low-rank subspace of self-supervised feature spaces. (3) The alignment should be supplied at an intermediate diffusion state and used purely as a training signal, so the inference pipeline is unchanged.

The mechanism is a small predictor whose orthogonal basis is parameterized by the residual textual signal and which projects the diffusion latent into a low-rank subspace of self-supervised visual features for supervision. We prove that the cross-modal information recoverable through this mechanism at a given rank is bounded by the spectral mass of the visual encoder's covariance in the top components, and under power-law spectra this capacity is largely reached at modest rank. The bound is a capacity result; the evidence that the recovered signal acts on binding is empirical (Sections~\ref{sec:exp:main} and~\ref{sec:exp:ablations}).


\paragraph{Contributions.} We make the following contributions.
\begin{itemize}
    \item We formalize the diagnosis of compositional failure as a problem of misalignment rather than missing information: modern T2I systems already condition on T5, but the diffusion objective does not provide a training signal that grounds T5's compositional structure in visual composition.
    \item We propose \textsc{Orca}, a training-time auxiliary loss that supplies this grounding signal by aligning the diffusion model's intermediate latent with a low-rank target derived from self-supervised visual features, with the readout subspace parameterized by the T5--CLIP residual.
    \item We prove a spectral upper bound on the cross-modal information \textsc{Orca} can recover at a given rank, and show that under the power-law spectra that self-supervised visual features exhibit, the information missed at rank $n$ decays polynomially in $n$.
    \item We evaluate \textsc{Orca} against vanilla, REPA, and REG baselines on three diffusion-transformer backbones, primarily on FID and GenEval~\citep{ghosh2023geneval}, with T2I-CompBench~\citep{huang2023t2icompbench} and a $512{\times}512$ run as further checks and PickScore~\citep{kirstain2023pickscore} and CLIPScore~\citep{hessel2021clipscore} as auxiliary metrics. The largest gains are on attribute binding, spatial relations, and multi-object prompts, and an input ablation shows that they depend on the T5--CLIP residual.
\end{itemize}

\section{Preliminaries}
\label{sec:background}

We review flow and diffusion models through the unified perspective of
stochastic interpolants~\citep{albergo2023stochastic, ma2024sit} and then
specialise to the rectified-flow setting used by
MMDiT~\citep{esser2024scaling}. Appendix~\ref{app:background} gives the
probability-flow ODE, the reverse SDE, the score-velocity correspondence, and
classifier-free guidance.

Let $x^* \sim p_{\mathrm{data}}(x)$ and $\epsilon \sim \mathcal{N}(0, I)$, and
consider the time-dependent process on $t \in [0,T]$
\begin{equation}
x_t \;=\; \alpha_t\, x^* \;+\; \sigma_t\, \epsilon, \qquad \alpha_0 = \sigma_T = 1,\; \alpha_T = \sigma_0 = 0,
\label{eq:interpolant}
\end{equation}
with $\alpha_t$ decreasing and $\sigma_t$ increasing in $t$. The associated
probability-flow ODE $\dot x_t = v(x_t,t)$ has marginal $p_t(x)$ at every $t$,
so data can be sampled by integrating it from $\epsilon$ at $t = T$ down to
$t = 0$. Its velocity field is approximated by a network
$v_\theta(x_t, t, c)$ conditioned on text $c$ and trained with the
velocity-matching objective
\begin{equation}
\mathcal{L}_{\mathrm{diff}}(\theta) \;:=\; \E_{x^*,\, \epsilon,\, t}\,\bigl\|\, v_\theta(x_t, t, c) \;-\; \dot\alpha_t\, x^* \;-\; \dot\sigma_t\, \epsilon\,\bigr\|^2.
\label{eq:velocity_loss}
\end{equation}
We operate in the latent space of a frozen pretrained
autoencoder~\citep{rombach2022high}, taking $x^* = \mathcal{E}(\mathrm{image})$
as the data variable. Following~\citep{ma2024sit, esser2024scaling} we adopt
the linear interpolant with $T = 1$, $\alpha_t = 1-t$, $\sigma_t = t$, which
reduces~\eqref{eq:interpolant} to $x_t = (1-t)\,x^* + t\,\epsilon$ and the
target velocity to $\epsilon - x^*$. Nothing in
Section~\ref{sec:method} depends on this choice: ORCA's auxiliary loss is
defined for any $\alpha_t, \sigma_t$ for which~\eqref{eq:velocity_loss} is
well posed, including DDPM~\citep{ho2020denoising}.

\paragraph{Text conditioning in MMDiT.}
The conditioning signal $c$ is supplied by two frozen pretrained text
encoders: a CLIP~\citep{radford2021learning} embedding
$z_y^C \in \R^{d_C}$ and a T5~\citep{raffel2020exploring} embedding
$z_y^T \in \R^{d_T}$ of a prompt $y$, with $c = (z_y^C, z_y^T)$. MMDiT
processes text and image tokens with two separate sets of weights joined by
full sequence-level attention, producing hidden states $h_t^{(\ell)}$ at each
block index $\ell$. We write $h_T \in \R^d$ for the hidden state at the
diffusion block that Section~\ref{sec:method} supervises against an external
visual target.
The auxiliary loss is applied independently to each image-token hidden state
at this block, with the readout basis $K(\Delta_y)$ shared across the tokens
of a given prompt; we suppress the token index throughout for readability.
\section{Method}
\label{sec:method}


We now describe ORCA, a training-time auxiliary objective that supplies the diffusion model with a grounding signal for compositional structure. Given an MMDiT diffusion model conditioned on CLIP and T5 text embeddings as in Section~\ref{sec:background}, ORCA introduces three components that together construct an auxiliary loss applied to the diffusion model's latent $h_T$. The components are: a residual textual signal (\S\ref{sec:method:residual}) that isolates the part of T5 not expressible in CLIP, a low-rank visual target (\S\ref{sec:method:target}) derived from a self-supervised visual encoder, and a QR map (\S\ref{sec:method:qr}) that produces an input-dependent orthonormal projection from $h_T$ to the visual target space. The auxiliary loss and full training objective are given in \S\ref{sec:method:objective}; the theoretical analysis follows in \S\ref{sec:method:theory}. Inference is unchanged from the MMDiT baseline (\S\ref{sec:method:inference}). Throughout, we denote the diffusion model's parameters by $\theta$ and ORCA's auxiliary parameters (the textual projection and the QR map) by $\phi$.

\subsection{Residual textual signal}
\label{sec:method:residual}

The first component isolates the textual information that compositional generation requires but contrastive alignment does not preserve. Given the CLIP text embedding $z_y^C \in \R^{d_C}$ and the T5 text embedding $z_y^T \in \R^{d_T}$, we introduce a single learnable linear projection $W \in \R^{d_C \times d_T}$ from T5 to CLIP's embedding space, and define the residual textual signal as
\begin{equation}
\Delta_y \;=\; W\, z_y^T \;-\; z_y^C \;\in\; \R^{d_C}.
\label{eq:delta_y}
\end{equation}
The projection $W$ is trained jointly with the rest of ORCA through the auxiliary objective; we apply no explicit regulariser to it. A trivial solution in which $W$ is chosen so that $W z_y^T \approx z_y^C$ (driving $\Delta_y$ towards zero on the training distribution) leaves the predictor downstream with no input signal and is therefore not an attractor of the loss; the input ablation in Section~\ref{sec:exp:ablations} confirms this empirically. We use $\Delta_y$ as a learned difference signal in CLIP's coordinate frame,
not as an orthogonal complement of CLIP. This signal serves as the input that
determines the geometry of ORCA's predictor and empirically induces
prompt-dependent readout subspaces that vary systematically with semantic
changes in the prompt.

\subsection{Low-rank visual target}
\label{sec:method:target}

The second component supplies the diffusion model with a target representation of the clean image, structured to match the spectral concentration of self-supervised visual features. Let $E_D$ be a frozen pretrained self-supervised visual encoder (DINOv2~\citep{oquab2024dinov2} in our experiments), producing features $v = E_D(x) \in \R^{d_D}$ for a clean image $x$. Before training, we compute the empirical mean $\bar v$ and the top-$n$ principal components of $\{v_i\}_{i=1}^N$ on the training corpus, and assemble the principal-component matrix $P_n \in \R^{n \times d_D}$. The visual target is the projection
\begin{equation}
z(x) \;=\; P_n\bigl(E_D(x) - \bar v\bigr) \;\in\; \R^{n}.
\label{eq:target}
\end{equation}
Both $P_n$ and $\bar v$ are frozen for the duration of training and receive no gradient. The bottleneck rank $n$ is small in practice ($n \ll d_D$); the spectral analysis of \S\ref{sec:method:theory} characterises how much a given rank can recover.

The construction has three properties relevant to what follows. First, the target is decorrelated and variance-ordered by spectral theorem: $\mathrm{Cov}(z(x)) = \mathrm{diag}(\lambda_1, \ldots, \lambda_n)$ where $\lambda_1 \geq \cdots \geq \lambda_n$ are the top-$n$ eigenvalues of the empirical covariance of DINO features. Second, the target has no learnable parameters, which removes representational collapse on the target side as a stationary point of optimisation without recourse to variance regularisers~\citep{bardes2022vicreg} or asymmetric stop-gradient schemes~\citep{grill2020byol}. Third, the variance of each target dimension is non-zero by construction, providing a stable supervision signal across the full bottleneck.

\subsection{QR map}
\label{sec:method:qr}

\begin{figure}
    \centering
    \includegraphics[width=\linewidth]{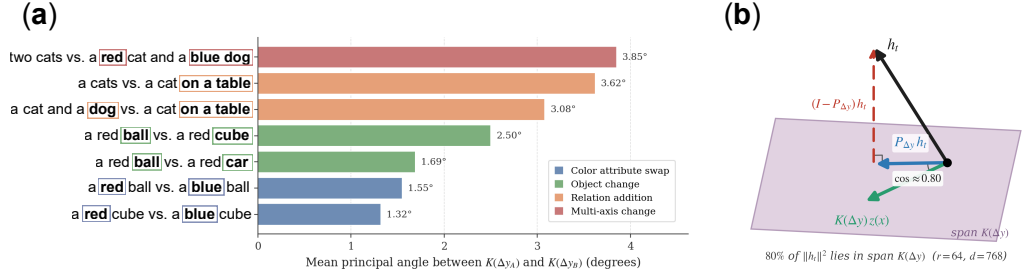}
    \caption{\textbf{(a)} Mean principal angle (in degrees) between the subspaces induced by paired prompts that differ along a controlled axis. Pairs are grouped by perturbation type: color attribute swap, object change, relation addition, and multi-axis change. The angles increase systematically with the semantic distance between prompts, with multi-axis changes producing the largest deviations and color swaps the smallest. The QR map is therefore input-dependent in a structured way: linguistically related prompts induce nearby readout subspaces, while linguistically distant prompts induce more separated ones.
  \textbf{(b)} Geometric illustration of the orthogonal decomposition of Equation~\eqref{eq:decomposition} with rank $r = 64$ and $d = 768$.}
    \label{fig:decomposition}
\end{figure}

The third component projects the diffusion model's latent $h_T \in \R^d$ into the $n$-dimensional target space, with the projection's geometry determined by the residual textual signal $\Delta_y$. We parameterise this projection as the orthonormalised output of a small MLP applied to $\Delta_y$. Concretely, an MLP $g_\phi : \R^{d_C} \to \R^{d \times n}$ produces a candidate matrix from the residual signal, and we orthonormalise its columns to obtain
\begin{equation}
K(\Delta_y) \;=\; \mathrm{QR}\bigl(g_\phi(\Delta_y)\bigr) \;\in\; \R^{d \times n}, \qquad K(\Delta_y)^\top K(\Delta_y) = I_n.
\label{eq:qr}
\end{equation}
We refer to this construction as the \emph{QR map}. The orthonormality constraint $K(\Delta_y)^\top K(\Delta_y) = I_n$ guarantees that the predicted target is a Euclidean projection of $h_T$ onto an $n$-dimensional subspace of $\R^d$ rather than an arbitrary linear combination, which keeps the auxiliary loss interpretable as a subspace alignment and stabilises the training dynamics. In practice we use Householder QR (\texttt{torch.linalg.qr}) with column normalisation and a small additive noise applied to the candidate matrix before the decomposition; implementation details, mixed-precision handling, and empirical orthonormality measurements are given in Appendix~\ref{app:qr_details}.

The predicted target is the projection of $h_T$ onto $K(\Delta_y)$:
\begin{equation}
\hat z(h_T, \Delta_y) \;=\; K(\Delta_y)^\top h_T \;\in\; \R^n.
\label{eq:prediction}
\end{equation}
The factorisation $\hat z = K(\Delta_y)^\top h_T$ separates the role of the text from the role of the image: the text determines \emph{which subspace of $\R^d$ to read out}, and the image fills in coordinates within that subspace. The next subsection makes this decomposition precise.
\subsection{The factorisation as an orthogonal decomposition of \texorpdfstring{$h_T$}{hT}}
\label{sec:method:factorisation}
 
The factorised form of the predictor in Equation~\eqref{eq:prediction} induces a natural decomposition of the diffusion model's intermediate latent that clarifies what ORCA's auxiliary loss does and does not constrain. Define the rank-$n$ orthogonal projector
\begin{equation}
P_{\Delta_y} \;:=\; K(\Delta_y)\, K(\Delta_y)^\top \;\in\; \R^{d \times d},
\label{eq:projector}
\end{equation}
which by construction satisfies $P_{\Delta_y}^2 = P_{\Delta_y}$, $P_{\Delta_y}^\top = P_{\Delta_y}$, and $\mathrm{rank}(P_{\Delta_y}) = n$. The latent $h_T$ admits the orthogonal decomposition
\begin{equation}
h_T \;=\; \underbrace{P_{\Delta_y}\, h_T}_{\,h_T^\parallel\,} \;+\; \underbrace{(I - P_{\Delta_y})\, h_T}_{\,h_T^\perp\,},
\label{eq:decomposition}
\end{equation}
where $h_T^\parallel$ lies in the $n$-dimensional subspace selected by the QR map for the prompt $y$, and $h_T^\perp$ lies in its orthogonal complement.
 
\paragraph{ORCA constrains only $h_T^\parallel$.}
A direct calculation shows
\begin{equation}
K(\Delta_y)^\top h_T \;=\; K(\Delta_y)^\top P_{\Delta_y} h_T \;=\; K(\Delta_y)^\top h_T^\parallel,
\end{equation}
because $K^\top (I - P_{\Delta_y}) = 0$ by orthonormality of $K$. The auxiliary loss in Equation~\eqref{eq:loss} is therefore a function of $h_T^\parallel$ alone:
\begin{equation}
\mathcal{L}_{\mathrm{ORCA}} \;=\; \E\!\bigl\|\,\mathrm{sg}[z(x)] \;-\; K(\Delta_y)^\top h_T^\parallel\,\bigr\|^2.
\end{equation}
Equivalently, $\mathcal{L}_{\mathrm{ORCA}}$ is invariant to any change in $h_T^\perp$ as shown in Figure \ref{fig:decomposition}. The decomposition in Equation~\eqref{eq:decomposition} is the unique (up to the choice of subspace, which is fixed by $K(\Delta_y)$) decomposition of $h_T$ that separates the component the auxiliary loss can shape from the component it cannot.  
 
\subsection{Auxiliary loss and training objective}
\label{sec:method:objective}

The auxiliary loss is the squared error between the predicted target and the frozen visual target, with the target detached from the computation graph:
\begin{equation}
\mathcal{L}_{\mathrm{ORCA}}(\theta, \phi) \;=\; \E_{(x, y) \sim p_{\mathrm{data}}, \, t \sim p_t}\!\left[\,\bigl\|\,\mathrm{sg}[z(x)] \;-\; K_\phi(\Delta_y)^\top h_T(x, y; \theta)\,\bigr\|^2\,\right],
\label{eq:loss}
\end{equation}
where $h_T(x, y; \theta)$ is MMDiT's latent for noisy input at time $t$, the stop-gradient $\mathrm{sg}[\cdot]$ blocks gradients into the visual encoder, and $K_\phi$ makes the dependence on $\phi$ explicit. The total objective combines the diffusion loss of Equation~\eqref{eq:velocity_loss} with the auxiliary loss:
\begin{equation}
\mathcal{L}(\theta, \phi) \;=\; \mathcal{L}_{\mathrm{diff}}(\theta) \;+\; \lambda \, \mathcal{L}_{\mathrm{ORCA}}(\theta, \phi),
\label{eq:total_loss}
\end{equation}
where $\lambda > 0$ is a single scalar hyperparameter. We find performance to be robust over approximately one order of magnitude of $\lambda$; see Section~\ref{sec:exp:ablations} and Appendix~\ref{app:additional:lambda}.

\subsection{Inference}
\label{sec:method:inference}

At inference the auxiliary parameters $\phi$ and the visual target $z(x)$ are not used. The diffusion model is sampled exactly as in the MMDiT baseline, integrating the velocity field $v_\theta(x_t, t, c)$ from $t=1$ to $t=0$ with classifier-free guidance as in Appendix~\ref{app:background:cfg}. ORCA therefore adds zero parameters, zero memory, and zero FLOPs to the inference pipeline.

\subsection{Theoretical analysis}
\label{sec:method:theory}

We now formalise what ORCA recovers and why a small rank $n$ is sufficient in practice. The argument has three steps. First, we establish a precise upper bound on the cross-modal information recoverable through the auxiliary loss at a given rank, in terms of the visual encoder's covariance spectrum (Theorem~\ref{thm:spectral-bound}). Second, we observe that self-supervised visual features on natural images have power-law spectra. Third, combining the two yields Proposition~\ref{prop:powerlaw}: ORCA at modest rank recovers a constant fraction of all available cross-modal information, with the residual decaying polynomially in $n$. Together they characterise what a rank-$n$ readout can recover; the rank ablation in Section~\ref{sec:exp:ablations} shows how performance varies with $n$ in practice. Full proofs are in Appendix~\ref{app:proofs}.

\paragraph{Setup.}
Let $V = E_D(X) \in \R^{d_D}$ denote the DINO feature of a random image $X$, with mean $\bar v$ and covariance $\Sigma_V = \sum_{i=1}^{d_D} \lambda_i u_i u_i^\top$ ordered by $\lambda_1 \geq \cdots \geq 0$. Let $Z = P_n(V - \bar v) \in \R^n$ be the visual target and $H \in \R^d$ the diffusion model's latent. Let $\hat Z(H)$ be any predictor of $Z$ from $H$.

\begin{theorem}[Spectral bound on recoverable information]
\label{thm:spectral-bound}
The maximum reduction in expected reconstruction error of $V$ achievable by any predictor $\hat Z(H)$ in ORCA's auxiliary objective is bounded above by the top-$n$ spectral mass of $\Sigma_V$:
\begin{equation}
\E\|V - \bar v\|^2 \;-\; \min_{\hat Z}\, \E\bigl\|\,V - \bar v - P_n^\top \hat Z(H)\,\bigr\|^2 \;\leq\; \sum_{i=1}^{n} \lambda_i,
\end{equation}
with equality when $\E[Z \mid H] = Z$ almost surely.
\end{theorem}

\begin{proposition}[Power-law concentration]
\label{prop:powerlaw}
Suppose $\lambda_i \leq C\, i^{-\alpha}$ for some $\alpha > 1$ and $C > 0$. Then the missing spectral mass beyond rank $n$ decays polynomially:
\begin{equation}
\sum_{i=n+1}^{\infty} \lambda_i \;\leq\; \frac{C}{(\alpha-1)\,n^{\alpha-1}}.
\end{equation}
Consequently the gap between recoverable information at rank $n$ and the total recoverable information is $O\bigl(n^{-(\alpha-1)}\bigr)$.
\end{proposition}

Self-supervised visual features on natural images empirically follow power laws of the form $\lambda_i \approx C\, i^{-\alpha}$ with $\alpha > 1$~\citep{ruderman1994statistics, bahri2024explaining, kaplan2020scaling, maloney2022solvable}, and under such spectra Theorem~\ref{thm:spectral-bound} and Proposition~\ref{prop:powerlaw} together say that most of the recoverable visual information is reachable at modest rank. This characterises ORCA's capacity-rank trade-off and explains why a method that operates in the principal subspace of the visual encoder is matched to the geometry of natural-image data.

\paragraph{Scope of the analysis.} Theorem~\ref{thm:spectral-bound} is a capacity result: it bounds how much of the visual target a rank-$n$ readout can recover, is agnostic to which information that is, and says nothing about optimisation. That ORCA's signal acts on binding is an empirical claim: the gains concentrate on the binding categories of GenEval (Tables~\ref{tab:geneval_breakdown} and~\ref{tab:geneval_full}), and the FID gain falls back to the level of encoder-only inputs when the readout subspace no longer depends on the T5--CLIP residual (Table~\ref{tab:residual}).

\definecolor{rowREPA}{HTML}{FCE4E4}   
\definecolor{rowREG}{HTML}{FBE8D2}    
\definecolor{rowORCA}{HTML}{E5F4E0}   
\definecolor{groupSep}{HTML}{B8B8B8}  

\begin{table*}[t]
\centering
\caption{\textbf{Main results on MS-COCO $256{\times}256$.}
Results are grouped by diffusion backbone. For each backbone, we report the
training checkpoint, FID-30K, and GenEval. $^\dagger$ORCA adds zero parameters
and FLOPs at inference. Best per backbone in bold.\quad
\colorbox{rowREPA}{\strut\,REPA\,}\;%
\colorbox{rowREG}{\strut\,REG\,}\;%
\colorbox{rowORCA}{\strut\,ORCA\,(ours)\,}}
\label{tab:main}

\footnotesize
\setlength{\tabcolsep}{4.2pt}
\renewcommand{\arraystretch}{1.20}

\begin{tabular*}{\textwidth}{@{\extracolsep{\fill}}
l
ccc
@{\hspace{5pt}}!{\color{groupSep}\vrule width 0.35pt}@{\hspace{5pt}}
ccc
@{\hspace{5pt}}!{\color{groupSep}\vrule width 0.35pt}@{\hspace{5pt}}
ccc
@{}}
\toprule
&
\multicolumn{3}{c}{\textbf{DiT-B/2} $(130\mathrm{M})$} &
\multicolumn{3}{c}{\textbf{DiT-L/2} $(458\mathrm{M})$} &
\multicolumn{3}{c}{\textbf{U-ViT-L} $(287\mathrm{M})$} \\
\cmidrule(lr){2-4}
\cmidrule(lr){5-7}
\cmidrule(lr){8-10}
\textbf{Setting}
& \textbf{Iter.} & \textbf{FID}$\downarrow$ & \textbf{GenEval}$\uparrow$
& \textbf{Iter.} & \textbf{FID}$\downarrow$ & \textbf{GenEval}$\uparrow$
& \textbf{Iter.} & \textbf{FID}$\downarrow$ & \textbf{GenEval}$\uparrow$ \\
\midrule

\multicolumn{10}{@{}l}{\emph{Vanilla diffusion training}} \\
\addlinespace[1pt]
Vanilla
& 100K & 35.49 & 0.123
& 100K & 27.77 & 0.162
& 100K & 38.72 & 0.130 \\
Vanilla
& 150K & 33.55 & 0.145
& 150K & 24.26 & 0.191
& 150K & 32.48 & 0.148 \\
Vanilla
& 200K & 29.95 & 0.167
& 200K & 24.71 & 0.195
& 200K & 30.95 & 0.161 \\
Vanilla
& 400K & 29.15 & 0.174
& 400K & 24.01 & 0.247
& 400K & 29.91 & 0.189 \\

\addlinespace[2pt]
\midrule
\multicolumn{10}{@{}l}{\emph{Representation-alignment baselines}} \\
\addlinespace[1pt]
\rowcolor{rowREPA}
$+$ REPA
& 250K & 23.85 & 0.214
& 400K & 20.05 & 0.275
& 400K & 23.31 & 0.221 \\
\rowcolor{rowREG}
$+$ REG
& 200K & 21.92 & 0.217
& 200K & 18.57 & 0.258
& 200K & 20.21 & 0.219 \\

\addlinespace[2pt]
\midrule
\multicolumn{10}{@{}l}{\emph{ORCA}} \\
\addlinespace[1pt]
\rowcolor{rowORCA}
$+$ \textbf{ORCA} (ours)$^\dagger$
& 100K & 28.31 & 0.180
& 100K & 20.51 & 0.238
& 100K & 25.47 & 0.218 \\
\rowcolor{rowORCA}
$+$ \textbf{ORCA} (ours)$^\dagger$
& 150K & 26.58 & 0.206
& 150K & 16.76 & 0.272
& 150K & 21.10 & 0.260 \\
\rowcolor{rowORCA}
$+$ \textbf{ORCA} (ours)$^\dagger$
& 250K & \textbf{21.06} & \textbf{0.265}
& 200K & \textbf{16.65} & \textbf{0.291}
& 200K & \textbf{19.62} & \textbf{0.271} \\

\bottomrule
\end{tabular*}
\end{table*}

\section{Experiments}
\label{sec:experiments}

\subsection{Setup}
\label{sec:exp:setup}

\paragraph{Dataset.}
We train on MS-COCO~\citep{lin2014coco} at $256{\times}256$ following
the standard text-to-image protocol used in prior
representation-alignment work~\citep{yu2024representation}. All
images are encoded into the latent space of the SD-1.5 VAE before
training, and captions are tokenised separately for CLIP and T5. The
PCA basis used for the visual target is estimated once on a fixed
subset of $5{,}000$ training images ($\sim\!1.28\text{M}$ patch
embeddings, sub-sampled to $100\text{K}$ for SVD efficiency) and
frozen for the remainder of training; rank $r{=}64$ captures $37.3\%$
of the cumulative spectral mass.
 
\paragraph{Architectures.}
We instantiate ORCA on three diffusion-transformer backbones at two
parameter scales to test that the gains do not depend on a specific
design or model size: {DiT-B/2}~\citep{peebles2023scalable}
($\sim\!130\text{M}$), {DiT-L/2}~\citep{peebles2023scalable}
($\sim\!458\text{M}$), and {U-ViT-L}~\citep{bao2023uvit}
($\sim\!287\text{M}$), the latter differs in skip-connection
topology. All backbones use identical conditioning, optimiser, and
schedule settings; the only difference between methods is the
auxiliary loss. Architectural details and the full hyperparameter
table are in Appendix~\ref{app:architecture:hyperparams}.
 
\paragraph{Baselines.}
We compare ORCA against (i) the {vanilla} backbone trained
with the diffusion objective alone, (ii)
{REPA}~\citep{yu2024representation}, and (iii) REG~\citep{reg2025}, the strongest
representation-alignment method we are aware of. All baselines use
the same backbone, optimiser, and schedule as ORCA; the only
difference is the auxiliary loss. For REPA and REG we use the authors'
released configuration. Baseline hyperparameters are reproduced
verbatim in Appendix~\ref{app:architecture:hyperparams}.
 
\paragraph{Metrics.}
We report {FID-30K}~\citep{heusel2017fid} for generation
quality, {GenEval}~\citep{ghosh2023geneval} for fine-grained
compositional alignment, {CLIPScore}~\citep{hessel2021clipscore}
(ViT-L/14) for text--image alignment, and
{PickScore}~\citep{kirstain2023pickscore} as a human-preference
proxy. All metrics use a fixed seed, $50$ DDIM sampling steps, and
CFG scale $s{=}1.5$ unless noted otherwise.
 
\paragraph{Training-step protocol.}
We report each method at the checkpoints shown in Table~\ref{tab:main}.
Vanilla baselines are reported up to $400\mathrm{K}$ steps; REPA and REG use
the released/default configurations at the listed checkpoints. For ORCA, we
report multiple checkpoints to characterise convergence speed. The rank ablation
runs are trained to $150\mathrm{K}$ steps; the layer-placement and $\lambda$
ablations are trained to $50\mathrm{K}$ steps.

\subsection{Main results: backbones and benchmarks}
\label{sec:exp:main}

Table~\ref{tab:main} reports the comparison on FID and GenEval. The pattern is
consistent across backbones: standard diffusion training improves slowly with
scale, while representation alignment substantially accelerates convergence.
On DiT-L/2, vanilla training reaches FID $24.01$ and GenEval $0.247$ after
$400\mathrm{K}$ steps, REPA reaches $20.05$ and $0.275$ at the same budget,
and REG reaches $18.57$ and $0.258$ at $200\mathrm{K}$. ORCA shifts this
trade-off further: at the same $200\mathrm{K}$ budget as REG it reaches FID
$\mathbf{16.65}$ and GenEval $\mathbf{0.291}$, an improvement of $10.3\%$ and
$12.8\%$ over the strongest matched-budget baseline, and it exceeds both
$400\mathrm{K}$ runs on both metrics while using half the training budget.
ORCA is therefore not only faster to optimise; it improves the
quality-composition trade-off. Changes in CLIPScore and PickScore are small
across all methods, consistent with prior observations that global preference
and text-image similarity scores are less sensitive to the binding errors
GenEval measures.

Table~\ref{tab:geneval_breakdown} drills into the per-category structure of
the GenEval gain at a matched $150\text{K}$ steps for all methods. The
improvements concentrate on the categories that require compositional
reasoning rather than single-object photorealism: Position
($\mathbf{2.9\times}$ over vanilla), Color attribution
($\mathbf{3.0\times}$), Two objects ($\mathbf{1.6\times}$), and Counting
($\mathbf{1.4\times}$), while single-object accuracy is already comparatively
saturated for all methods and improves only modestly ($1.2\times$). This
distribution of gains is the empirical signature our diagnosis predicts: the
bottleneck is binding rather than photorealism, and ORCA acts on the former.
Per-backbone breakdowns are in Appendix~\ref{app:additional:geneval_full},
and Appendix~\ref{app:additional:projection} tracks the factorisation of
\S\ref{sec:method:factorisation} over training.

\paragraph{Harder prompts and higher resolution.}
Two further checks ask whether the gains hold beyond GenEval's templates and
beyond $256{\times}256$ (Table~\ref{tab:compbench_coco512}). On
T2I-CompBench~\citep{huang2023t2icompbench}, whose prompts are longer and
combine several attributes and relations, the DiT-L/2 ORCA checkpoint at
$150\text{K}$ steps improves over the vanilla checkpoint at the same budget
in every category and raises the six-category mean from $0.221$ to
$\mathbf{0.244}$. This evaluates existing checkpoints without retraining,
and absolute spatial scores stay low for both models. On MS-COCO at
$512{\times}512$, ORCA leads REPA in FID at $150\text{K}$ steps
($\mathbf{17.16}$ vs.\ $20.33$); this run had not converged, so we read it
as a trend rather than a final number.

\begin{table}[t]
\centering
\caption{\textbf{GenEval per-task breakdown} on DiT-L/2 at
$150\text{K}$ steps for all methods (matched to the rank-ablation
budget). Gains concentrate on tasks that require compositional
reasoning (\textit{Two objects}, \textit{Counting},
\textit{Position}, \textit{Color attribution}); single-object
accuracy is comparatively saturated and improves only modestly.
We additionally report ORCA at multiple ranks $r$ to show the gain
is not idiosyncratic to one rank choice. Best per column in
{bold}.}
\label{tab:geneval_breakdown}
\small
\setlength{\tabcolsep}{5pt}
\renewcommand{\arraystretch}{1.05}
\begin{tabular}{l c c c c c c c}
\toprule
\textbf{Method} & \textbf{Single} & \textbf{Two} & \textbf{Count}
& \textbf{Colors} & \textbf{Position} & \textbf{Color attr.} & \textbf{Overall} \\
\midrule
Vanilla         & 0.516 & 0.038 & 0.116 & 0.431 & 0.025 & 0.023 & 0.191 \\
REPA            & 0.638 & 0.040 & 0.116 & 0.620 & 0.050 & 0.030 & 0.249 \\
\midrule
ORCA ($r{=}16$) & 0.622 & 0.043 & 0.119 & 0.588 & 0.043 & 0.038 & 0.242 \\
ORCA ($r{=}32$) & 0.613 & \textbf{0.083} & 0.181 & 0.582 & 0.050 & \textbf{0.068} & 0.263 \\
ORCA ($r{=}64$) & 0.634 & 0.061 & 0.159 & 0.638 & \textbf{0.073} & \textbf{0.068} & 0.272 \\
ORCA ($r{=}128$) & \textbf{0.669} & 0.068 & \textbf{0.194} & \textbf{0.641} & 0.058 & 0.048 & \textbf{0.279} \\
\bottomrule
\end{tabular}
\end{table}

\begin{table}[t]
\centering
\caption{\textbf{Harder prompts and higher resolution} (DiT-L/2,
$150\text{K}$ steps). \textbf{Left:} T2I-CompBench scores of the vanilla and
ORCA checkpoints at the same budget, evaluated without retraining (higher is
better). \textbf{Right:} FID-30K on MS-COCO at $512{\times}512$, before
convergence (lower is better).}
\label{tab:compbench_coco512}
\small
\setlength{\tabcolsep}{3.5pt}
\renewcommand{\arraystretch}{1.08}
\begin{minipage}[t]{0.74\linewidth}
\centering
\begin{tabular}{l c c c c c c c}
\toprule
\textbf{Method} & \textbf{Color} & \textbf{Shape} & \textbf{Texture}
& \textbf{Spatial} & \textbf{Non-spat.} & \textbf{Complex} & \textbf{Mean} \\
\midrule
Vanilla & 0.3030 & 0.2470 & 0.3195 & 0.0118 & 0.2457 & 0.2010 & 0.221 \\
\rowcolor{rowORCA}
ORCA & \textbf{0.3312} & \textbf{0.2564} & \textbf{0.3626} & \textbf{0.0234}
     & \textbf{0.2652} & \textbf{0.2243} & \textbf{0.244} \\
\bottomrule
\end{tabular}
\end{minipage}%
\hfill
\begin{minipage}[t]{0.22\linewidth}
\centering
\begin{tabular}{l c}
\toprule
\textbf{Method} & \textbf{FID}~$\downarrow$ \\
\midrule
REPA & 20.33 \\
\rowcolor{rowORCA}
ORCA & \textbf{17.16} \\
\bottomrule
\end{tabular}
\end{minipage}
\end{table}

\section{Ablations}
\label{sec:exp:ablations}

We first isolate the mechanism, the residual textual signal that parameterises
the readout subspace (Table~\ref{tab:residual}), then sweep subspace rank,
alignment block, and loss weight $\lambda$ (Figure~\ref{fig:fid_ablations};
DiT-L/2 at $150\text{K}$ steps for rank and $50\text{K}$ for the others).
Per-metric tables and the loss and target ablations are in
Appendix~\ref{app:additional}.

\begin{table}[t]
\centering
\caption{\textbf{The QR map must read the learned residual.} Only the input
to the QR map varies (DiT-L/2, $150\text{K}$ steps, $r{=}64$, all else at
defaults). All rows come from one sweep and are comparable with each other;
this sweep's ORCA run scores $17.32$, against $16.76$ for the same nominal
configuration in Table~\ref{tab:rank}. $\Delta$FID is relative to the ORCA
row.}
\label{tab:residual}
\small
\setlength{\tabcolsep}{6pt}
\renewcommand{\arraystretch}{1.08}
\begin{tabular}{l l c c}
\toprule
\textbf{Input to the QR map} & \textbf{What it removes}
& \textbf{FID}~$\downarrow$ & \textbf{$\Delta$FID} \\
\midrule
\rowcolor{rowORCA}
$\Delta_y = W z_y^T - z_y^C$ \;(ORCA)
  & --
  & \textbf{17.32} & -- \\
\addlinespace[2pt]
\multicolumn{4}{l}{\textit{Textual signal: which embedding drives the readout}} \\
\midrule
T5 only ($z_y^T$)
  & the CLIP subtraction
  & 19.37 & $+2.05$ \\
CLIP only ($z_y^C$)
  & the language-model stream
  & 20.03 & $+2.71$ \\
Concatenation $[z_y^T; z_y^C]$
  & the explicit factorisation
  & 22.44 & $+5.12$ \\
\addlinespace[2pt]
\multicolumn{4}{l}{\textit{Controls: is any input enough?}} \\
\midrule
Fixed identity $K$
  & prompt-dependence of the subspace
  & 19.72 & $+2.40$ \\
Frozen random $\Delta_y$
  & the learned mapping $g_\phi$
  & 19.53 & $+2.21$ \\
\bottomrule
\end{tabular}
\end{table}

\paragraph{The residual textual signal.}
Table~\ref{tab:residual} varies only the input to the QR map. The learned
residual $\Delta_y$ reaches FID $17.32$, and every alternative falls back to a
$19$--$20$ band: T5 alone ($19.37$) and CLIP alone ($20.03$) are worse than
their difference, and concatenating both without the subtraction is worst
($22.44$), so the gain comes from removing the CLIP-expressible component
rather than from extra textual capacity. The two controls separate the signal
from the machinery: a fixed identity basis, which removes prompt-dependence,
gives $19.72$, and a frozen random input, which removes the learned mapping
$g_\phi$, gives $19.53$. Having an input is therefore not sufficient; the
readout subspace has to be a learned function of the part of T5 that CLIP
discards. This also shows that the degenerate solution of
\S\ref{sec:method:residual} is not reached: were $W z_y^T$ to collapse onto
$z_y^C$, ORCA would behave like the frozen-random control.

\paragraph{Subspace rank.}
FID is minimised at $r{=}64$ ($16.76$) and degrades at both smaller and
larger ranks, while GenEval improves monotonically with $r$ and peaks at
$r{=}128$ ($0.279$). We adopt $r{=}64$ as a Pareto-efficient default: it
pays a $2.5\%$ relative GenEval cost for an $8.4\%$ relative FID gain over
$r{=}128$. CLIPScore and PickScore are flat across the sweep, so the rank
choice does not interact strongly with surface-level image quality. The
non-monotonicity is not predicted by the capacity argument of
\S\ref{sec:method:theory}, which bounds what is recoverable at a given rank
but says nothing about optimisation; we return to this in
\S\ref{sec:conclusion}. Per-metric numbers are in
Appendix~\ref{app:additional:rank}, Table~\ref{tab:rank}.

\paragraph{Alignment-block placement.}
Performance is U-shaped in the depth at which the auxiliary loss is applied:
intermediate depth (block~$8$ of $24$) outperforms both very early
(block~$1$) and very late (block~$24$) placement, with block~$24$ collapsing
entirely (FID~$64.16$, GenEval~$0.054$). The collapse at the last block
indicates that the representation immediately before the output head is too
task-specific to benefit from the auxiliary signal, consistent with REPA's
observation that self-supervised structure is most useful where the diffusion
model's representations are still abstract~\citep{yu2024representation}. Full
numbers are in Appendix~\ref{app:additional:layer}, Table~\ref{tab:layer}.

\paragraph{Loss coefficient.}
Performance is robust over approximately one order of magnitude of $\lambda$
around the default $\lambda{=}1.0$; the curve is U-shaped, with both very
small and very large $\lambda$ degrading FID
(Appendix~\ref{app:additional:lambda}, Table~\ref{tab:lambda}).

\section{Related work}
\label{sec:related}

\paragraph{Compositional text-to-image generation.}
Compositional failures are documented by benchmarks such as
GenEval~\citep{ghosh2023geneval} and
T2I-CompBench~\citep{huang2023t2icompbench}, and persist in recent
systems~\citep{compalign2025, infinitybeyond2025}. Diagnostic work locates
part of the problem in CLIP's text encoder, which often behaves as a bag of
words~\citep{thrush2022winoground, yuksekgonul2023clipbow, lewis2024doesclip},
and proposes fixes at the level of the text
embedding~\citep{zarei2024improving, magnet2024}.
Other remedies act at inference time by steering
cross-attention~\citep{chefer2023attend, feng2023structure}, or at training
time by supervising token-level attention with object
masks~\citep{wang2024tokencompose} or by conditioning on a large language
model through a learned connector~\citep{hu2024ella}. ORCA is a training-time
method, but it leaves the encoders and the attention maps untouched: it
supervises an intermediate latent against a visual target, so nothing changes
at inference.

\paragraph{Representation alignment for diffusion training.}
REPA~\citep{yu2024representation} showed that aligning diffusion hidden states
with frozen self-supervised features speeds up training, and follow-up work
refines the recipe: REG~\citep{reg2025} entangles the latents with a class
token from a foundation model, SARA~\citep{chen2025sara} adds structural and
adversarial alignment terms, U-REPA~\citep{tian2025urepa} adapts the method to
U-Nets, and HASTE~\citep{wang2025haste} also aligns attention maps and
switches the alignment off partway through training, with results on MS-COCO
text-to-image as well as ImageNet.
SoftREPA~\citep{lee2025softrepa} is the closest text-conditioned method: it
fine-tunes a pretrained text-to-image model with a contrastive objective
through soft text tokens, which remain in the model at inference. In all of
these, the alignment target does not depend on the prompt. Appendix~\ref{app:positioning}
compares the methods along these axes.

\paragraph{What is inherited and what is new.}
Aligning diffusion hidden states with a frozen self-supervised encoder is
REPA's idea, and DINOv2, PCA, Householder QR, and the MSE loss are standard
tools. ORCA contributes the misalignment diagnosis, the T5--CLIP residual
as the textual signal, a readout subspace that is a function of the prompt, so
that the text selects which part of the visual target the latent is aligned
to, and the capacity bound of \S\ref{sec:method:theory}. To our knowledge ORCA is the first
representation-alignment method whose readout depends on the prompt, and
Table~\ref{tab:residual} indicates that this dependence is where its gain
comes from.

\begin{figure}[t]
\centering
\includegraphics[width=0.9\linewidth]{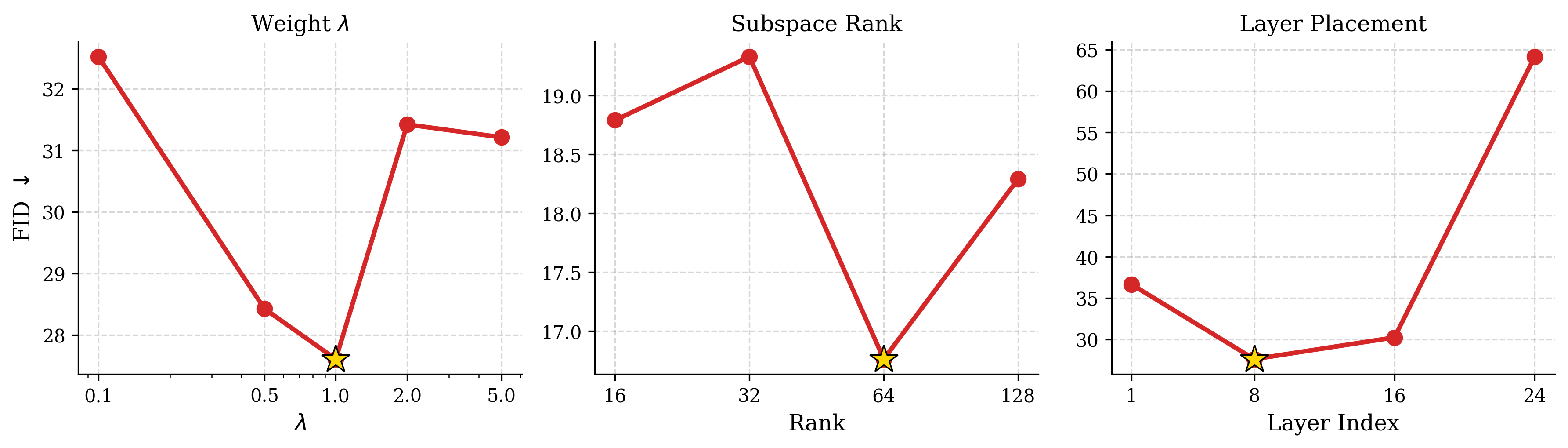}
\caption{\textbf{ORCA design-choice sensitivity (FID-30K).} Each panel sweeps
a single hyperparameter while the other two are held at their default; stars
mark the default configuration adopted in the main results. \textbf{Left:}
the auxiliary-loss weight $\lambda$ has a clean optimum at $\lambda{=}1.0$.
\textbf{Centre:} the bottleneck rank $r$ is non-monotonic, with $r{=}64$ a
sharp minimum. \textbf{Right:} intermediate depth (block~$8$ of $24$) is
strongly preferred; alignment at the final block collapses the model
(FID~$64.16$). Full per-metric tables are in Appendix~\ref{app:additional}.}
\label{fig:fid_ablations}
\end{figure}

\section{Conclusion and limitations}
\label{sec:conclusion}

We presented \textsc{Orca}, a training-time auxiliary objective that supplies
text-to-image diffusion models with an explicit grounding signal for
compositional structure: it aligns an intermediate diffusion latent with a
low-rank principal subspace of self-supervised visual features, read out
through a QR map driven by the T5--CLIP residual, and leaves inference
unchanged. Across three diffusion-transformer backbones, \textsc{Orca}
improves FID and GenEval over vanilla training, REPA, and REG; on DiT-L/2 it
reaches FID $\mathbf{16.65}$ and GenEval $\mathbf{0.291}$ at $200\text{K}$
steps, exceeding the strongest $400\text{K}$ baseline at half the training
cost. The gains concentrate on attribute binding, spatial relations, and
multi-object prompts, the failure modes our diagnosis targets, and the input
ablation shows that they depend on the readout subspace being a learned
function of the T5--CLIP residual.

\paragraph{Limitations.}
Our experiments use MS-COCO at $256{\times}256$, plus one $512{\times}512$
run, with backbones of 130M to 458M parameters trained from scratch; we have
not fine-tuned a large pretrained system such as Stable Diffusion~3 or FLUX,
or tested a model conditioned on a vision-language model, whose language
stream would take the place of T5. Each configuration is trained once, so the
bootstrap error bars of Appendix~\ref{app:additional:errorbars} capture
sampling but not training variance, and two ORCA runs at the same nominal
configuration differ by $0.56$ FID (Tables~\ref{tab:residual}
and~\ref{tab:rank}). Of the related methods we compare only with REPA and
REG; HASTE, which also reports MS-COCO text-to-image results, is the most
direct missing baseline. Theorem~\ref{thm:spectral-bound} bounds capacity but
does not explain why the recovered signal is compositional, nor the
non-monotonic dependence on rank. Finally, better control over image
generation can also make misleading images easier to produce; we release code
but no pretrained generators.

\begin{ack}
This work was carried out using the computing infrastructure and GPU
resources of the Wellcome Sanger Institute. We acknowledge core funding from
the Wellcome Trust (WT220540/Z/20/A).
We thank all members of the Lotfollahi Group for their valuable feedback and
suggestions during internal discussions.

\paragraph{Competing interests.} The authors declare no competing interests.
\end{ack}
\bibliography{neurips_2026}
\bibliographystyle{unsrt}

\clearpage
\appendix

\appendix

\section{Additional Preliminaries}
\label{app:background}

This appendix provides additional detail on the stochastic-interpolant framework summarised in Section~\ref{sec:background}.

\subsection{Probability flow ODE and reverse SDE}
\label{app:background:pfode}

For the interpolant $x_t = \alpha_t x^* + \sigma_t \epsilon$ defined in Equation~\eqref{eq:interpolant}, the marginal $p_t$ at time $t$ is the law of $x_t$ under $x^* \sim p_{\mathrm{data}}$ and $\epsilon \sim \mathcal{N}(0, I)$. There exists a probability-flow ordinary differential equation
\begin{equation}
\dot x_t \;=\; v(x_t, t),
\label{app:eq:pfode}
\end{equation}
whose marginal at time $t$ matches $p_t$~\citep{albergo2023stochastic, ma2024sit}. The velocity field is the conditional expectation
\begin{equation}
v(x, t) \;=\; \E[\dot x_t \mid x_t = x] \;=\; \dot \alpha_t \E[x^* \mid x_t = x] + \dot \sigma_t \E[\epsilon \mid x_t = x],
\label{app:eq:velocity}
\end{equation}
which is the regression target of the velocity-matching loss in Equation~\eqref{eq:velocity_loss}. There also exists a reverse stochastic differential equation
\begin{equation}
d x_t \;=\; v(x_t, t)\, dt \;-\; \tfrac{1}{2} w_t\, s(x_t, t)\, dt \;+\; \sqrt{w_t}\, d\bar w_t,
\end{equation}
with diffusion coefficient $w_t \geq 0$ and score function $s(x, t) = \nabla_x \log p_t(x)$. The score and velocity are related by
\begin{equation}
s(x, t) \;=\; \frac{\alpha_t\, v(x, t) - \dot\alpha_t\, x}{\sigma_t (\dot\alpha_t \sigma_t - \alpha_t \dot\sigma_t)},
\end{equation}
so a learned velocity model induces a score model and vice versa, and either ODE solvers (acting on Equation~\eqref{app:eq:pfode}) or SDE solvers can be used at inference~\citep{ma2024sit}.

\subsection{Specialisation to rectified flow and DDPM}
\label{app:background:specializations}

Setting $\alpha_t = 1 - t$, $\sigma_t = t$ on $[0, 1]$ gives the linear interpolant of rectified flow~\citep{lipman2023flow, esser2024scaling}, with target velocity $\dot\alpha_t x^* + \dot\sigma_t \epsilon = \epsilon - x^*$. Setting $\alpha_t = \sqrt{\bar\alpha_t}$, $\sigma_t = \sqrt{1 - \bar\alpha_t}$ for a discretised cosine or linear schedule $\bar\alpha_t$ recovers the DDPM setting~\citep{ho2020denoising}. ORCA's auxiliary loss (Equation~\eqref{eq:loss}) does not depend on the choice of $\alpha_t, \sigma_t$ and therefore applies in either case.

\subsection{Classifier-free guidance}
\label{app:background:cfg}

At inference, MMDiT applies classifier-free guidance~\citep{ho2022classifierfree} by training the velocity model to predict both the conditional and an unconditional velocity (the unconditional case is obtained by replacing the text embeddings with a learned null token at a fixed dropout rate during training). At sampling time, the guided velocity is
\begin{equation}
\hat v(x_t, t, c) \;=\; v_\theta(x_t, t, \varnothing) + s\bigl(v_\theta(x_t, t, c) - v_\theta(x_t, t, \varnothing)\bigr),
\end{equation}
where $s \geq 1$ is the guidance scale. ORCA does not modify the guidance procedure: the auxiliary loss is applied to $h_T$ obtained from the conditional pass and is not used during the unconditional pass.

\section{Proofs}
\label{app:proofs}

\subsection{Proof of Theorem~\ref{thm:spectral-bound}}
\label{app:proofs:thm1}

We restate the theorem and provide a self-contained proof.

\begin{theorem}[Restated]
Let $V \in \R^{d_D}$ be a random visual feature with mean $\bar v$ and covariance $\Sigma_V = \sum_{i=1}^{d_D} \lambda_i u_i u_i^\top$ with $\lambda_1 \geq \lambda_2 \geq \cdots \geq 0$. Let $P_n$ be the matrix whose rows are $u_1, \ldots, u_n$, and define $Z = P_n(V - \bar v) \in \R^n$. For any predictor $\hat Z(H)$ from a context $H$,
\begin{equation*}
\E\|V - \bar v\|^2 \;-\; \min_{\hat Z}\, \E\bigl\|V - \bar v - P_n^\top \hat Z(H)\bigr\|^2 \;\leq\; \sum_{i=1}^{n} \lambda_i,
\end{equation*}
with equality iff $\E[Z \mid H] = Z$ almost surely.
\end{theorem}

\begin{proof}
For any $v \in \R^{d_D}$ and any $\hat z \in \R^n$, the orthogonal decomposition $v - \bar v = P_n^\top P_n (v - \bar v) + (I - P_n^\top P_n)(v - \bar v)$ satisfies, by Pythagoras (since $P_n^\top P_n$ is the orthogonal projector onto $\mathrm{span}(u_1, \ldots, u_n)$),
\begin{equation}
\|v - \bar v - P_n^\top \hat z\|^2 \;=\; \|P_n(v - \bar v) - \hat z\|^2 \;+\; \|(I - P_n^\top P_n)(v - \bar v)\|^2.
\label{app:eq:pythagoras}
\end{equation}
The second term is independent of $\hat z$. Taking expectations and using $\Sigma_V = \E[(V - \bar v)(V - \bar v)^\top]$,
\begin{equation*}
\E\bigl\|(I - P_n^\top P_n)(V - \bar v)\bigr\|^2 \;=\; \mathrm{tr}\bigl[(I - P_n^\top P_n) \Sigma_V\bigr] \;=\; \sum_{i=n+1}^{d_D} \lambda_i,
\end{equation*}
the standard Eckart--Young identity. With $Z = P_n(V - \bar v)$, taking expectations of Equation~\eqref{app:eq:pythagoras} gives
\begin{equation*}
\E\|V - \bar v - P_n^\top \hat Z(H)\|^2 \;=\; \E\|Z - \hat Z(H)\|^2 \;+\; \sum_{i=n+1}^{d_D} \lambda_i.
\end{equation*}
Subtracting both sides from $\E\|V - \bar v\|^2 = \sum_{i=1}^{d_D} \lambda_i$,
\begin{equation*}
\E\|V - \bar v\|^2 - \E\|V - \bar v - P_n^\top \hat Z(H)\|^2 \;=\; \sum_{i=1}^{n} \lambda_i \;-\; \E\|Z - \hat Z(H)\|^2.
\end{equation*}
The right-hand side is maximised over $\hat Z$ when $\E\|Z - \hat Z(H)\|^2$ is minimised. By the law of total variance, this minimum is attained by the minimum-mean-squared-error predictor $\hat Z^\star(H) = \E[Z \mid H]$, with non-negative residual error and zero residual error iff $\E[Z \mid H] = Z$ almost surely.
\end{proof}

\subsection{Proof of Proposition~\ref{prop:powerlaw}}
\label{app:proofs:prop1}

\begin{proposition}[Restated]
Suppose $\lambda_i \leq C i^{-\alpha}$ for some $\alpha > 1$ and $C > 0$. Then
\begin{equation*}
\sum_{i=n+1}^{\infty} \lambda_i \;\leq\; \frac{C}{(\alpha-1) n^{\alpha-1}}.
\end{equation*}
\end{proposition}

\begin{proof}
Since $i \mapsto i^{-\alpha}$ is decreasing for $i \geq 1$, the integral test gives
\begin{equation*}
\sum_{i=n+1}^{\infty} i^{-\alpha} \;\leq\; \int_{n}^{\infty} x^{-\alpha}\, dx \;=\; \frac{n^{-(\alpha-1)}}{\alpha - 1}.
\end{equation*}
Multiplying by $C$ and using $\lambda_i \leq C i^{-\alpha}$ yields the claim.
\end{proof}

\subsection{Tightness of the bound under linear-Gaussian assumptions}
\label{app:proofs:tightness}

If $(V, H)$ is jointly Gaussian and $V$ depends on $H$ through a linear map plus independent noise, the conditional expectation $\E[Z \mid H]$ is itself a linear function of $H$. Under this assumption the optimal predictor in Theorem~\ref{thm:spectral-bound} is realised within the predictor class used by ORCA (a linear projection of $h_T$ onto a subspace determined by $\Delta_y$), and the bound becomes tight whenever $H$ contains all the linear information about $Z$. This justifies treating the bound as a target rather than as a loose ceiling in the regime of practical interest.

\section{Architecture and Hyperparameters}
\label{app:architecture}

\subsection{QR map: numerical details}
\label{app:qr_details}
 
The QR map of Equation~\eqref{eq:qr} produces an orthonormal basis
$K(\Delta_y) \in \R^{d \times n}$ from the output of the basis MLP
$g_\phi$. We use the standard Householder QR decomposition as
implemented by \texttt{torch.linalg.qr}, with three implementation
choices that ensure stability under mixed-precision training and at
the start of training when $g_\phi$ is small.
 
\paragraph{Stability at the start of training.}
The output linear layer of $g_\phi$ is initialised with Xavier
initialisation at gain $0.01$ for the weights and zero for the bias,
so $g_\phi(\Delta_y)$ has small Frobenius norm at training step zero
but is not exactly the zero matrix. To ensure that the QR
decomposition is well-defined throughout training, we apply two
preprocessing steps to the candidate matrix before invoking
\texttt{torch.linalg.qr}: (i)~we normalise each of its $n$ columns to
unit Euclidean norm, and (ii)~we add isotropic Gaussian noise of
magnitude $\sigma = 10^{-6}$ per entry. The first step prevents
column scales from drifting and provides a uniform conditioning
target; the second step pushes the matrix away from rank-deficient
configurations and prevents divisions through near-zero diagonal
entries of $R$ in the QR backward.
 
\paragraph{Mixed-precision handling.}
The remainder of the diffusion model is trained in fp16 mixed
precision, but the QR computation is held in an fp32 island: we cast
the input to fp32, disable autocast for the basis MLP and the QR
step, and cast the resulting orthonormal basis $K(\Delta_y)$ back to
the original (fp16) dtype before it is multiplied by the diffusion
hidden state. This avoids the well-known instability of QR under
fp16 arithmetic without paying its cost on the rest of the network,
since the QR step is small relative to the diffusion forward pass.
 
\paragraph{Empirical stability.}
We did not observe the large gradient norms or loss spikes sometimes
reported when backpropagating through QR at near rank-deficient
inputs. Throughout training, we monitor the orthonormality error
$\|K(\Delta_y)^\top K(\Delta_y) - I_n\|_{\max}$ (the element-wise
maximum absolute deviation from the identity) and observe values
below $5 \times 10^{-5}$ from step zero onwards. The combination of
column normalisation, additive noise, and fp32 QR is sufficient to
keep the orthonormality error small and the gradients well-behaved
across the entire $400\text{K}$-step training run.
 
\paragraph{Backward pass.}
The backward pass through \texttt{torch.linalg.qr} is implemented
analytically in PyTorch and involves the inverse of the
upper-triangular factor $R$. In our setting $R$ remains
well-conditioned because the column-normalisation step bounds the
column scales of the input, and the additive noise prevents exact
rank deficiency; the lower bound on the diagonal entries of $R$ is
in turn bounded away from zero in expectation. Gradient flow into
$g_\phi$ is stable in practice.

\subsection{ORCA module architecture}
\label{app:architecture:orca}
 
We summarise the ORCA module in Table~\ref{tab:architecture}. The
module adds three trainable components on top of the diffusion
backbone: the textual projection $W$, the basis MLP $g_\phi$, and
a QR orthonormalisation step. All are removed at inference and
contribute zero parameters and FLOPs to the deployed model.
 
\begin{table}[h]
\centering
\caption{ORCA module architecture. The textual projection $W$ is a single linear
layer mapping the T5 caption embedding into the $d_C$-dimensional CLIP space.
The basis MLP $g_\phi$ is a 2-layer MLP with LayerNorm and GELU; its output
layer is initialised with Xavier gain $0.01$ and zero bias, producing a small
pre-QR candidate matrix at initialisation. The QR step and the PCA projection
contribute no learnable parameters.}
\label{tab:architecture}
\small
\begin{tabular}{@{}lll@{}}
\toprule
Component & Specification & Parameters \\
\midrule
Textual projection $W$            & Linear($d_T \to d_C$)                                       & $d_T \cdot d_C + d_C$ \\
Basis MLP $g_\phi$, hidden      & Linear($d_C \to h_g$) + LayerNorm + GELU                    & $d_C h_g + h_g + 2 h_g$ \\
Basis MLP $g_\phi$, output      & Linear($h_g \to d \cdot n$), Xavier init (gain $0.01$)      & $h_g \cdot d \cdot n + d \cdot n$ \\
Orthonormalisation                & QR (\texttt{torch.linalg.qr}) on $d \times n$ matrix        & 0 (deterministic) \\
PCA projection $P_n$, $\bar v$    & Precomputed from frozen DINOv2                              & 0 (no gradient) \\
\bottomrule
\end{tabular}
\end{table}
 
In our default configuration, $h_g = d_C$ (the basis MLP hidden width
equals the CLIP dimension), so the parameter count of $g_\phi$ is
dominated by the output layer ($h_g \cdot d \cdot n$).
 
\subsection{Hyperparameters}
\label{app:architecture:hyperparams}

We list the hyperparameters used for our main experiments in
Table~\ref{tab:hyperparams}. Values are reported for the DiT-L/2
backbone at MS-COCO $256\times 256$ resolution; values for other
configurations are in Section~\ref{app:additional}.

\begin{table}[h]
\centering
\caption{Hyperparameters for ORCA training. The auxiliary-loss weight
$\lambda$ and rank $n$ are the two hyperparameters specific to ORCA;
all others are inherited from the diffusion backbone.}
\label{tab:hyperparams}
\small
\begin{tabular}{@{}lll@{}}
\toprule
Hyperparameter & Value & Notes \\
\midrule
\multicolumn{3}{l}{\emph{Diffusion training (inherited)}} \\
Optimiser                       & AdamW              & $\beta_1 = 0.9$, $\beta_2 = 0.999$ \\
Learning rate                   & $1 \times 10^{-4}$ & cosine schedule with linear warmup \\
Warmup steps                    & $5{,}000$  & \\
Batch size (global)             & $256$ & \\
Total training iterations       & $400{,}000$        & for the baselines result \\
Method training iterations       & $200{,}000$        & for our method's result \\

Weight decay                    & $0.01$ & \\
EMA decay                       & $0.9999$ & \\
Precision                 & bf16/fp16 mixed (fp32 island for QR)              & \\
\midrule
\multicolumn{3}{l}{\emph{ORCA-specific}} \\
Auxiliary-loss weight $\lambda$ & $1.0$              & ablated in Section~\ref{sec:exp:ablations} \\
Bottleneck rank $n$             & $64$               & ablated over $\{16, 32, 64, 128\}$ \\
Alignment block $M$             & $8$                & out of $24$ blocks; ablated in App.~\ref{app:additional:layer} \\
Visual encoder                  & DINOv2-L/14 & $d_D = 1024$ \\
PCA corpus size                 & $\sim$$5{,}000$ images, $\sim$$100{,}000$ patch vectors & sampled from training set \\
Basis MLP hidden dim $h_g$      & $1024$ & \\
$g_\phi$ output init & Xavier gain $0.01$, zero bias & small pre-QR candidate matrix at initialisation\\
Orthonormalisation & QR (\texttt{torch.linalg.qr}) &
column normalisation plus pre-QR noise \\
 Noise $\sigma$ & $10^{-6}$          & Additive Gaussian in App.~\ref{app:qr_details} \\
\midrule
\multicolumn{3}{l}{\emph{Encoders (frozen)}} \\
CLIP text encoder & CLIP ViT-L/14   & $d_C = 768$ \\
T5 text encoder   & T5-XL     & $d_T = 2048$ \\
\bottomrule
\end{tabular}
\end{table}

\subsection{Compute and runtime}
\label{app:architecture:compute}

Each experiment was run on $4 \times$ A100 GPUs. Training the MMDiT-L baseline for $400$K iterations takes approximately 1 to 2 days; adding ORCA increases per-step wall-clock time by less than 7\%.

\section{Additional Results}
\label{app:additional}

This appendix reports ablations and analyses that did not fit in the
main paper. All experiments use DiT-L/2, DiT-B/2, and U-ViT-L on MS-COCO $256{\times}256$
unless otherwise noted; the default ORCA configuration is rank
$r{=}64$, alignment block $M{=}8$, auxiliary-loss weight
$\lambda{=}1.0$, frozen DINOv2 PCA target with MSE loss.


\subsection{Full per-task GenEval breakdown}
\label{app:additional:geneval_full}

Table~\ref{tab:geneval_full} reports the per-category GenEval breakdown for the main configurations in Table~\ref{tab:main}. Across all three backbones, the same
qualitative pattern emerges: ORCA's gains over the vanilla baseline
concentrate on the categories that require compositional reasoning
rather than single-object photorealism.

Comparing the strongest ORCA configuration to the $400\text{K}$
vanilla baseline on each backbone, the largest multiplicative gains
are on {Color attribution} ($\mathbf{3.4\times}$ on DiT-B/2,
$\mathbf{2.6\times}$ on DiT-L/2, $\mathbf{3.9\times}$ on U-ViT-L),
{Two objects} ($\mathbf{2.4\times}$, $\mathbf{1.8\times}$,
$\mathbf{2.1\times}$ respectively), and {Position}
($\mathbf{2.8\times}$, $\mathbf{1.7\times}$, $\mathbf{1.3\times}$).
{Single object} accuracy, the only category that does not
require binding two semantic attributes, improves more modestly
($1.09$--$1.27\times$). The exception to the pattern is Counting on
DiT-L/2, where ORCA at $200\text{K}$ ($0.194$) stays below vanilla
training at $400\text{K}$ ($0.219$). This distribution of gains is consistent
across backbone families and matches the prediction of our
diagnosis: representation alignment via ORCA shifts the bottleneck
away from binding errors rather than from photorealism.
\begin{table}[t]
\centering
\caption{\textbf{Per-task GenEval breakdown} for the main configurations in Table~\ref{tab:main}. ``Single'' = single object, ``Two'' = two objects,
``Count'' = counting, ``Colors'' = single-object color, ``Position'',
``Color attr.'' = color attribution.}
\label{tab:geneval_full}
\small
\setlength{\tabcolsep}{3.5pt}
\renewcommand{\arraystretch}{1.05}
\begin{tabular}{l c c c c c c c c}
\toprule
\textbf{Model} & \textbf{Iter.} & \textbf{Single} & \textbf{Two} & \textbf{Count}
& \textbf{Colors} & \textbf{Position} & \textbf{Color attr.} & \textbf{Overall} \\
\midrule
\multicolumn{9}{l}{\emph{DiT-B/2}} \\
DiT-B/2                  & 100K & 0.381 & 0.023 & 0.072 & 0.250 & 0.008 & 0.005 & 0.123 \\
DiT-B/2                  & 150K & 0.397 & 0.033 & 0.097 & 0.301 & 0.020 & 0.020 & 0.145 \\
DiT-B/2                  & 200K & 0.475 & 0.030 & 0.100 & 0.356 & 0.020 & 0.020 & 0.167 \\
DiT-B/2                  & 400K & 0.491 & 0.020 & 0.106 & 0.378 & 0.028 & 0.020 & 0.174 \\
\,$+$ REPA               & 250K & 0.559 & 0.058 & 0.097 & 0.497 & 0.038 & 0.038 & 0.214 \\
\,$+$ REG               & 200K & 0.561 & 0.061 & 0.099 & 0.501 & 0.039 & 0.041 & 0.217 \\
\,$+$ ORCA               & 100K & 0.447 & 0.038 & 0.100 & 0.436 & 0.025 & 0.035 & 0.180 \\
\,$+$ ORCA               & 150K & 0.497 & 0.043 & 0.138 & 0.476 & 0.040 & 0.040 & 0.206 \\
\,$+$ ORCA               & 250K & \textbf{0.625} & 0.048 & 0.153 & \textbf{0.620} & \textbf{0.078} & \textbf{0.068} & \textbf{0.265} \\
\midrule
\multicolumn{9}{l}{\emph{DiT-L/2}} \\
DiT-L/2                  & 100K & 0.456 & 0.028 & 0.084 & 0.362 & 0.020 & 0.020 & 0.162 \\
DiT-L/2                  & 150K & 0.516 & 0.038 & 0.116 & 0.431 & 0.025 & 0.023 & 0.191 \\
DiT-L/2                  & 200K & 0.500 & 0.035 & 0.106 & 0.481 & 0.033 & 0.013 & 0.195 \\
DiT-L/2                  & 400K & 0.603 & 0.040 & 0.219 & 0.540 & 0.043 & 0.035 & 0.247 \\
\,$+$ REPA               & 400K & \textbf{0.703} & 0.058 & 0.108 & \textbf{0.673} & 0.050 & 0.058 & 0.275 \\
\,$+$ REG                & 200K & 0.653 & 0.035 & 0.150 & 0.644 & 0.033 & 0.035 & 0.258 \\
\,$+$ ORCA               & 100K & 0.606 & 0.038 & 0.141 & 0.561 & 0.035 & 0.048 & 0.238 \\
\,$+$ ORCA               & 150K & 0.634 & 0.061 & 0.159 & 0.638 & 0.073 & 0.068 & 0.272 \\
\,$+$ ORCA               & 200K & 0.656 & \textbf{0.073} & \textbf{0.194} & 0.657 & \textbf{0.075} & \textbf{0.090} & \textbf{0.291} \\
\midrule
\multicolumn{9}{l}{\emph{U-ViT-L}} \\
U-ViT-L                  & 100K & 0.381 & 0.013 & 0.066 & 0.295 & 0.013 & 0.015 & 0.130 \\
U-ViT-L                  & 150K & 0.400 & 0.018 & 0.100 & 0.340 & 0.013 & 0.015 & 0.148 \\
U-ViT-L                  & 200K & 0.466 & 0.028 & 0.088 & 0.348 & 0.018 & 0.018 & 0.161 \\
U-ViT-L                  & 400K & 0.534 & 0.038 & 0.109 & 0.399 & 0.038 & 0.015 & 0.189 \\
\,$+$ REPA                & 400K & 0.553 & 0.035 & 0.126 & 0.544 & 0.033 & 0.035 & 0.221 \\
\,$+$ REG                & 200K & 0.553 & 0.028 & 0.116 & 0.582 & 0.020 & 0.018 & 0.219 \\
\,$+$ ORCA               & 100K & 0.566 & 0.040 & 0.100 & 0.505 & 0.030 & 0.068 & 0.218 \\
\,$+$ ORCA               & 150K & 0.616 & 0.073 & 0.159 & 0.585 & \textbf{0.053} & \textbf{0.073} & 0.260 \\
\,$+$ ORCA               & 200K & \textbf{0.666} & \textbf{0.081} & \textbf{0.188} & \textbf{0.585} & 0.048 & 0.058 & \textbf{0.271} \\
\bottomrule
\end{tabular}
\end{table}

\subsection{Subspace rank}
\label{app:additional:rank}

Table~\ref{tab:rank} reports the effect of the bottleneck rank $r$ on
DiT-L/2 at $150\text{K}$ steps. FID and GenEval both improve as $r$
grows from $16$ to $64$, plateau around $r{\in}\{64,128\}$, and FID
begins to degrade above this range. 

\begin{table}[h]
\centering
\caption{\textbf{Effect of subspace rank $r$} on DiT-L/2 at $150\mathrm{K}$ steps.
FID is best at $r=64$, while GenEval is highest at $r=128$.
Best per column in bold.}
\label{tab:rank}
\small
\setlength{\tabcolsep}{8pt}
\renewcommand{\arraystretch}{1.05}
\begin{tabular}{c c c c c}
\toprule
{Rank $r$} & {FID}~$\downarrow$ & {GenEval}~$\uparrow$
& {CLIP-L/14}~$\uparrow$ & {PickScore}~$\uparrow$ \\
\midrule
16          & 18.79          & 0.242          & \textbf{0.1076} & \textbf{17.16} \\
32          & 19.33          & 0.263          & 0.1063          & 17.12 \\
\textbf{64} & \textbf{16.76} & 0.272          & 0.1042          & 17.11 \\
128         & 18.29          & \textbf{0.279} & 0.1045          & 17.14 \\
\bottomrule
\end{tabular}
\end{table}

\subsection{Alignment block placement}
\label{app:additional:layer}

Table~\ref{tab:layer} reports performance as a function of the
diffusion block at which the auxiliary loss is applied, with all other
settings held fixed at $r{=}64$, $\lambda{=}1.0$. Performance is
U-shaped in depth: aligning at intermediate depth (block~$8$ of $24$)
is best on FID and GenEval, while alignment at the very last block
collapses (FID${=}64.16$, GenEval${=}0.054$). This indicates that the
diffusion model's representation immediately before the output head is
too task-specific to benefit from the auxiliary signal, consistent
with REPA's observation that self-supervised structure is most useful
where the model's representations are still abstract
\citep{yu2024representation}. We use block~$8$ throughout.

\begin{table}[h]
\centering
\caption{\textbf{Layer placement of the auxiliary loss} on DiT-L/2 at
$50\text{K}$ training steps. Best per column in {bold}.}
\label{tab:layer}
\small
\setlength{\tabcolsep}{8pt}
\renewcommand{\arraystretch}{1.05}
\begin{tabular}{l c c c c}
\toprule
\textbf{Layer}        & \textbf{FID}~$\downarrow$ & \textbf{GenEval}~$\uparrow$
                      & \textbf{CLIP-B/32}~$\uparrow$ & \textbf{PickScore}~$\uparrow$ \\
\midrule
1  (very early)       & 36.63          & 0.130          & 0.2159          & 17.33 \\
\textbf{8} (default)  & \textbf{27.61} & \textbf{0.165} & \textbf{0.2172} & 17.25 \\
16 (middle)           & 30.23          & 0.162          & 0.2165          & 17.23 \\
24 (last)             & 64.16          & 0.054          & 0.2156          & \textbf{17.40} \\
\bottomrule
\end{tabular}
\end{table}

\subsection{Auxiliary-loss weight $\lambda$}
\label{app:additional:lambda}

Table~\ref{tab:lambda} reports a sweep of the auxiliary-loss weight
$\lambda$ at $50\text{K}$ training steps. Performance is robust over
roughly one order of magnitude around $\lambda{=}1.0$; very small
($\lambda{=}0.1$) and very large ($\lambda{=}5.0$) values both degrade
FID. The PickScore column is essentially flat across the sweep,
consistent with the observation in the main paper that PickScore is
dominated by surface aesthetics rather than compositional fidelity.

\begin{table}[h]
\centering
\caption{\textbf{Effect of the auxiliary-loss weight $\lambda$} on DiT-L/2 at $50\mathrm{K}$ steps.
The $\lambda=1.0$ configuration is the default used for the ORCA runs reported in
Table~1. Best per column in bold.}
\label{tab:lambda}
\small
\setlength{\tabcolsep}{6pt}
\renewcommand{\arraystretch}{1.05}
\begin{tabular}{c c c c c}
\toprule
$\boldsymbol{\lambda}$ & \textbf{FID}~$\downarrow$ & \textbf{GenEval}~$\uparrow$
                       & \textbf{CLIP}~$\uparrow$ & \textbf{PickScore}~$\uparrow$ \\
\midrule
0.1           & 32.52          & 0.135          & 0.2163          & \textbf{17.29} \\
0.5           & 28.43          & \textbf{0.182} & 0.2164          & 17.26 \\
\textbf{1.0}  & \textbf{27.61} & 0.165          & \textbf{0.2172} & 17.25 \\
2.0           & 31.42          & 0.172          & 0.2151          & 17.22 \\
5.0           & 31.21          & 0.157          & 0.2159          & 17.25 \\
\bottomrule
\end{tabular}
\end{table}

\subsection{Loss function and target construction}
\label{app:additional:loss_target}

Table~\ref{tab:loss_target} compares the design choices for the visual
target and alignment loss. A learnable linear projection trained
end-to-end with MSE collapses to a trivial solution (the target shrinks
to a small subspace and predictor accuracy becomes meaningless). Adding
VICReg regularisation \citep{bardes2022vicreg} prevents the collapse but
recovers only partially. Replacing the target with the frozen
top-$r$ DINOv2 PCA, the choice used in ORCA, decisively
outperforms both. A Gaussian negative-log-likelihood (NLL) loss on the
same frozen target underperforms MSE.

The frozen target is the key design choice: it removes target-side
collapse as a stationary point of optimisation, eliminating the need
for VICReg or other anti-collapse regularisers entirely.

\begin{table}[h]
\centering
\caption{\textbf{Loss function and target construction} on DiT-L/2.
Frozen PCA with MSE is the strongest configuration. Best per column in
{bold}.}
\label{tab:loss_target}
\small
\setlength{\tabcolsep}{6pt}
\renewcommand{\arraystretch}{1.05}
\begin{tabular}{l l c c}
\toprule
\textbf{Target} & \textbf{Loss} & \textbf{FID}~$\downarrow$
                & \textbf{GenEval}~$\uparrow$ \\
\midrule
Learnable linear (no reg.)     & MSE                  & collapsed       & collapsed     \\
Learnable linear + VICReg      & MSE                  & 29.49           & 0.145  \\
\textbf{Frozen PCA (ours)}     & \textbf{MSE}         & \textbf{16.76}  & \textbf{0.272} \\
Frozen PCA                     & Gaussian NLL         & 23.31           & 0.257  \\
\bottomrule
\end{tabular}
\end{table}

\subsection{Projection metrics over training}
\label{app:additional:projection}

Figure~\ref{fig:projection_appendix} reports three quantities measured
on validation samples over the course of training:

\begin{enumerate}
  \item the energy ratio $\|P_{\Delta y} h_t\|^2 / \|h_t\|^2$, which
        measures what fraction of the diffusion hidden state lies in
        $\operatorname{span} K(\Delta_y)$;
  \item the cosine similarity between the aligned component
        $P_{\Delta y} h_t$ and the lifted DINOv2 PCA target
        $K(\Delta_y) z(x)$;
  \item the cosine similarity between the orthogonal-complement
        residual $(I-P_{\Delta y}) h_t$ and the same target.
\end{enumerate}

The energy ratio rises from $\sim$$0.58$ at initialisation to
$\sim$$0.80$ within $\sim$$80\text{K}$ steps and remains stable
thereafter. The aligned-component cosine rises from $\sim$$0.52$ to
$\sim$$0.80$ on a similar timescale. The residual cosine remains at
zero throughout: this is forced by construction
($P_{\Delta y}(I-P_{\Delta y}) = 0$), and we report it as a sanity
check on the implementation rather than a learned property.

\begin{figure}[h]
\centering
\includegraphics[width=0.8\linewidth]{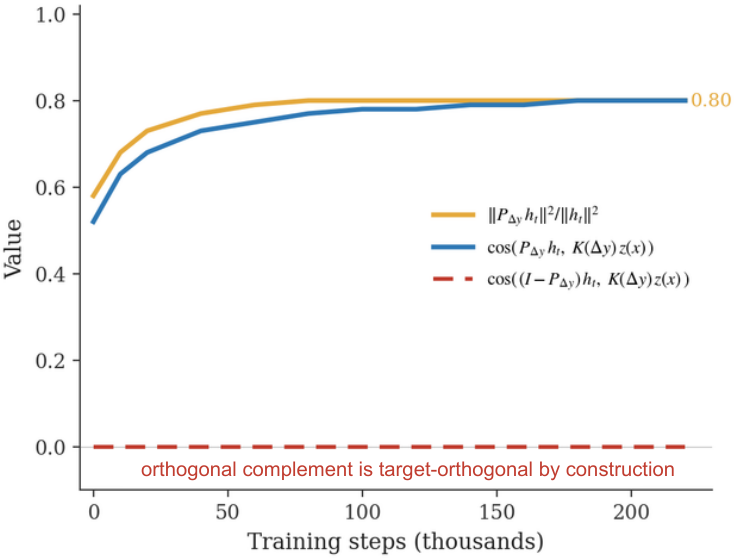}
\caption{\textbf{Projection metrics over training.} Energy ratio (orange),
aligned-component cosine (blue), and residual-target cosine (red, dashed).
The aligned component captures $80\%$ of $\|h_t\|^2$ in a rank-$64$ subspace
and reaches cosine $0.80$ with the DINOv2 PCA target. The residual-target
cosine is zero by construction and is included only as an implementation sanity
check; it should not be interpreted as evidence that the orthogonal complement
contains no visual information.}
\label{fig:projection_appendix}
\end{figure}


\section{Bootstrap error bars on main metrics}
\label{app:additional:errorbars}

Compute constraints did not permit retraining each model
configuration with multiple seeds (each $400\text{K}$-step run takes
1 to 2 days on $4 \times$ A100 GPUs). To still characterise the
statistical reliability of the headline FID and GenEval numbers in
Table~\ref{tab:main}, we estimate sampling-level uncertainty
\emph{post-hoc} via bootstrap resampling of the cached evaluation
artefacts.

\paragraph{FID bootstrap.}
For each configuration we resample with replacement from the cached
pool of $30\text{K}$ generated Inception-V3 features and recompute
FID against the fixed real-feature reference distribution
($30\text{K}$ MS-COCO validation features, identical across all
methods). The real distribution is held fixed; only the generated
features are resampled, since the source of metric uncertainty under
this protocol is which $30\text{K}$ samples were drawn from the
generator. We report the unbiased ($\mathrm{ddof}{=}1$) standard
deviation across $100$ bootstrap replicates with seed $0$.

\paragraph{GenEval bootstrap.}
For each configuration we resample with replacement from the
per-image binary correctness scores ($553$ prompts $\times$ $4$
images per prompt $=2{,}212$ trials) and report the unbiased
standard deviation of the resulting GenEval averages across
$1{,}000$ bootstrap replicates with seed $0$.

\paragraph{Reproducibility.}
The bootstrap script, fixed RNG seeds, cached real features, and
per-config cached generated features are released with the code so
that the standard deviations in Table~\ref{tab:errorbars} can be
exactly reproduced.

\begin{table}[h]
\centering
\caption{\textbf{Bootstrap error bars on FID and GenEval} for the
DiT-L/2 configurations of Table~\ref{tab:main}. FID standard
deviations come from $100$ bootstrap resamples of the $30\text{K}$
generated feature pool; GenEval standard deviations come from
$1{,}000$ bootstrap resamples of the per-image correctness scores.
All standard deviations are $1$-$\sigma$, unbiased
($\mathrm{ddof}{=}1$).}
\label{tab:errorbars}
\small
\setlength{\tabcolsep}{6pt}
\renewcommand{\arraystretch}{1.10}
\begin{tabular}{l c c c}
\toprule
\textbf{Model} & \textbf{Iter.} &
\textbf{FID}~$\downarrow$ & \textbf{GenEval}~$\uparrow$ \\
\midrule
\multicolumn{4}{l}{\emph{DiT-L/2}} \\
DiT-L/2                              & 100K & $27.77 \pm 0.16$ & $0.162 \pm 0.008$ \\
DiT-L/2                              & 150K & $24.26 \pm 0.17$ & $0.191 \pm 0.008$ \\
DiT-L/2                              & 200K & $24.71 \pm 0.16$ & $0.195 \pm 0.009$ \\
DiT-L/2                              & 400K & $24.01 \pm 0.16$ & $0.247 \pm 0.009$ \\
\,$+$ REPA                           & 400K & $20.05 \pm 0.12$ & $0.275 \pm 0.009$ \\
\,$+$ \textbf{ORCA} (ours)           & 100K & $20.51 \pm 0.13$ & $0.238 \pm 0.009$ \\
\,$+$ \textbf{ORCA} (ours)           & 150K & $16.76 \pm 0.10$ & $0.272 \pm 0.009$ \\
\,$+$ \textbf{ORCA} (ours)           & 200K & $\mathbf{16.65 \pm 0.11}$ & $\mathbf{0.291 \pm 0.010}$ \\
\bottomrule
\end{tabular}
\end{table}
\section{Positioning against representation-alignment methods}
\label{app:positioning}

Table~\ref{tab:positioning} compares ORCA with the representation-alignment
methods discussed in Section~\ref{sec:related} along the axes that matter for
compositional text-to-image training: what the diffusion model is aligned to,
whether that target depends on the prompt, and the training regime.

\begin{table}[h]
\centering
\caption{\textbf{Representation-alignment methods compared.} Only ORCA reads
out the visual target in a subspace selected by the prompt.}
\label{tab:positioning}
\small
\setlength{\tabcolsep}{5pt}
\renewcommand{\arraystretch}{1.1}
\begin{tabular}{l p{6.2cm} c l}
\toprule
\textbf{Method} & \textbf{Alignment target} & \textbf{Prompt-dependent?} & \textbf{Regime} \\
\midrule
REPA~\citep{yu2024representation} & fixed self-supervised patch features & no & from scratch \\
REG~\citep{reg2025} & self-supervised features, with a class token entangled into the latents & no & from scratch \\
SARA~\citep{chen2025sara} & fixed features: patch-wise, autocorrelation, and adversarial terms & no & from scratch \\
U-REPA~\citep{tian2025urepa} & fixed ViT features at the U-Net middle stage, plus a manifold loss & no & from scratch \\
HASTE~\citep{wang2025haste} & fixed features and attention maps, switched off mid-training & no & from scratch \\
SoftREPA~\citep{lee2025softrepa} & contrastive text-image objective through soft text tokens & no & fine-tuning \\
\rowcolor{rowORCA}
ORCA (ours) & low-rank PCA of self-supervised features, read out in a subspace set by $\Delta_y$ & yes & from scratch \\
\bottomrule
\end{tabular}
\end{table}

\paragraph{When ORCA applies.}
ORCA is designed for training a text-conditioned diffusion transformer from
scratch when compositional binding is the target. It adds no parameters or
FLOPs at inference and less than $7\%$ per-step training time
(Appendix~\ref{app:architecture:compute}). SoftREPA instead adapts an existing
model and keeps its soft tokens at inference. Methods developed for
class-conditional generation do not target binding, since a single class
label has nothing to bind.

\clearpage
\newpage

\section*{NeurIPS Paper Checklist}

\begin{enumerate}

\item {\bf Claims}
    \item[] Question: Do the main claims made in the abstract and introduction accurately reflect the paper's contributions and scope?
    \item[] Answer: \answerYes{} 
    \item[] Justification: All the claims are discussed in the section \ref{sec:method} and section \ref{sec:experiments}

\item {\bf Limitations}
    \item[] Question: Does the paper discuss the limitations of the work performed by the authors?
    \item[] Answer: \answerYes{} 
    \item[] Justification: Limitations are discussed in the final section (Section~\ref{sec:conclusion}, paragraph ``Limitations''), covering scale, single-seed training, missing baselines, and the scope of the theory. 

\item {\bf Theory assumptions and proofs}
    \item[] Question: For each theoretical result, does the paper provide the full set of assumptions and a complete (and correct) proof?
    \item[] Answer: \answerYes{} 
    \item[] Justification: Assumptions are stated with each result in Section~\ref{sec:method:theory}, and complete proofs are in Appendix~\ref{app:proofs}.

    \item {\bf Experimental result reproducibility}
    \item[] Question: Does the paper fully disclose all the information needed to reproduce the main experimental results of the paper to the extent that it affects the main claims and/or conclusions of the paper (regardless of whether the code and data are provided or not)?
    \item[] Answer: \answerYes{} 
    \item[] Justification: all the details of experiments and their setup are already being discussed in section \ref{sec:exp:setup}, \ref{app:architecture:compute}, and section \ref{app:architecture:hyperparams} for hyperparameters. 

\item {\bf Open access to data and code}
    \item[] Question: Does the paper provide open access to the data and code, with sufficient instructions to faithfully reproduce the main experimental results, as described in supplemental material?
    \item[] Answer: \answerYes{} 
    \item[] Justification:  We use public data (MS-COCO), and the training code will be released publicly with instructions to reproduce the main results.

\item {\bf Experimental setting/details}
    \item[] Question: Does the paper specify all the training and test details (e.g., data splits, hyperparameters, how they were chosen, type of optimizer) necessary to understand the results?
    \item[] Answer: \answerYes{} 
    \item[] Justification: all the details of experiments and their setup are already being discussed in section \ref{sec:exp:setup}, \ref{app:architecture:compute}, and section \ref{app:architecture:hyperparams} for hyperparameters. 

\item {\bf Experiment statistical significance}
    \item[] Question: Does the paper report error bars suitably and correctly defined or other appropriate information about the statistical significance of the experiments?
    \item[] Answer: \answerYes{}
\item[] Justification:FID error bars are estimated by resampling
with replacement from the $30\text{K}$ generated Inception-V3 feature pool ($100$ bootstrap replicates, seed $0$); GenEval error bars come from resampling with replacement over the per-image correctness scores ($553$ prompts $\times$ $4$ images $=2{,}212$ trials, $1{,}000$ bootstrap replicates, seed $0$). The full methodology and sanity checks are in Appendix~\ref{app:additional:errorbars}.
 

\item {\bf Experiments compute resources}
    \item[] Question: For each experiment, does the paper provide sufficient information on the computer resources (type of compute workers, memory, time of execution) needed to reproduce the experiments?
    \item[] Answer: \answerYes{} 
    \item[] Justification: You can find the details in section \ref{app:architecture:compute}
    
\item {\bf Code of ethics}
    \item[] Question: Does the research conducted in the paper conform, in every respect, with the NeurIPS Code of Ethics \url{https://neurips.cc/public/EthicsGuidelines}?
    \item[] Answer: \answerYes{} 
    \item[] Justification: This paper conforms in every respect, with the NeurIPS Code of Ethics. 

\item {\bf Broader impacts}
    \item[] Question: Does the paper discuss both potential positive societal impacts and negative societal impacts of the work performed?
    \item[] Answer: \answerYes{} 
    \item[] Justification: ORCA is a training-time method for improving text-to-image
generation quality and compositional fidelity. Positive impacts include more
faithful controllable generation and reduced prompt ambiguity. Potential
negative impacts include improved generation of misleading or synthetic visual
content. We do not release pretrained image-generation checkpoints, and the
method is evaluated on established public research datasets and backbones. Potential misuse is also noted in Section~\ref{sec:conclusion}.
    
\item {\bf Safeguards}
    \item[] Question: Does the paper describe safeguards that have been put in place for responsible release of data or models that have a high risk for misuse (e.g., pre-trained language models, image generators, or scraped datasets)?
    \item[] Answer: \answerNA{} 
    \item[] Justification: ORCA is a training-time auxiliary loss; we release no pretrained generative models, scraped datasets, or other artefacts of dual-use concern. The diffusion backbones and benchmarks used are established public ones (DiT, U-ViT, MS-COCO).

\item {\bf Licenses for existing assets}
    \item[] Question: Are the creators or original owners of assets (e.g., code, data, models), used in the paper, properly credited and are the license and terms of use explicitly mentioned and properly respected?
    \item[] Answer: \answerYes{} 
    \item[] Justification: All third-party assets used in this paper are
established public research artefacts and are cited at first use.
We use MS-COCO~\citep{lin2014coco} (CC BY 4.0 for annotations) for
training data; CLIP ViT-L/14~\citep{radford2021learning} (MIT
license) and T5-v1.1~\citep{raffel2020exploring} (Apache 2.0) as frozen
text encoders; DINOv2~\citep{oquab2024dinov2} (Apache 2.0) as the
visual target encoder; and the DiT~\citep{peebles2023scalable} and
U-ViT~\citep{bao2023uvit} backbones (MIT licenses). Our REPA
baseline reuses the authors' released configuration~\citep{yu2024representation}
(MIT license).

\item {\bf New assets}
    \item[] Question: Are new assets introduced in the paper well documented and is the documentation provided alongside the assets?
    \item[] Answer: \answerYes{}
    \item[] Justification: We release the ORCA training code, including the basis MLP and QR map implementation, training scripts for the three backbones, and the configurations needed to reproduce the main results in Table ~\ref{tab:main}. The code is released under the MIT license alongside a README documenting installation, dependencies, and training commands.

\item {\bf Crowdsourcing and research with human subjects}
    \item[] Question: For crowdsourcing experiments and research with human subjects, does the paper include the full text of instructions given to participants and screenshots, if applicable, as well as details about compensation (if any)? 
    \item[] Answer: \answerNA{}
    \item[] Justification: The paper does not involve crowdsourcing or
    research with human subjects.

\item {\bf Institutional review board (IRB) approvals or equivalent for research with human subjects}
    \item[] Question: Does the paper describe potential risks incurred by study participants, whether such risks were disclosed to the subjects, and whether Institutional Review Board (IRB) approvals (or an equivalent approval/review based on the requirements of your country or institution) were obtained?
    \item[] Answer: \answerNA{}
    \item[] Justification: The paper does not involve crowdsourcing or research with human subjects.

\item {\bf Declaration of LLM usage}
    \item[] Question: Does the paper describe the usage of LLMs if it is an important, original, or non-standard component of the core methods in this research? Note that if the LLM is used only for writing, editing, or formatting purposes and does \emph{not} impact the core methodology, scientific rigor, or originality of the research, declaration is not required.
\item[] Answer: \answerNA{}
\item[] Justification: LLMs were used only for writing assistance (editing and formatting). They were not used as a component of the core methodology and do not impact the scientific rigour or originality of the research.

\end{enumerate}

\end{document}